\documentclass{article}

\usepackage[utf8]{inputenc}
\usepackage[T1]{fontenc}
\usepackage{amsmath,amssymb,amsfonts}
\usepackage{algorithmic}
\usepackage{algorithm}
\usepackage{graphicx}
\usepackage{booktabs}
\usepackage{hyperref}
\usepackage{xcolor}
\usepackage{subcaption}
\usepackage{multirow}
\usepackage{microtype}
\usepackage{nicefrac}
\usepackage{amsthm}
\usepackage{bbm}
\usepackage{mathtools}
\usepackage{makecell}

\usepackage[accepted]{icml2026}

\theoremstyle{plain}
\newtheorem{theorem}{Theorem}[section]

\newtheorem{lemma}[theorem]{Lemma}
\newtheorem{corollary}[theorem]{Corollary}
\theoremstyle{definition}

\theoremstyle{remark}

\DeclareMathOperator{\supp}{supp}

\icmltitlerunning{Expander Sparse Autoencoders: Parameter-Efficient Dictionaries for Mechanistic Interpretability}

\begin{document}

\twocolumn[
\icmltitle{Expander Sparse Autoencoders: \\Parameter-Efficient Dictionaries for Mechanistic Interpretability}



  \icmlsetsymbol{equal}{*}

  \begin{icmlauthorlist}
    \icmlauthor{Rodrigo Mendoza-Smith}{equal,yyy}
  \end{icmlauthorlist}

  \icmlaffiliation{yyy}{Independent Researcher}

  \icmlcorrespondingauthor{Rodrigo Mendoza-Smith}{rms@isotropic.sh}

  \icmlkeywords{Machine Learning, ICML}

  \vskip 0.3in
]



\printAffiliationsAndNotice{}  

\begin{abstract}
Sparse autoencoders (SAEs) decompose internal activations of neural networks into sparse linear combinations of learned features by fitting an overcomplete dictionary $\mathbf{W}\in\mathbb{R}^{m\times n}$ with $m<n$, and inferring a sparse code $\mathbf{x}\in\mathbb{R}^n$ from $\mathbf{h}\approx\mathbf{W}\mathbf{x}$.
This inference problem closely resembles the canonical setup of compressed sensing, but dense decoders requires $\mathcal{O}(mn)$ learned values, which becomes costly at large feature counts.
We introduce \textbf{Expander SAEs}: TopK SAEs whose decoder and tied encoder are supported on a left-\(d\)-regular expander mask with $d\ll m$, learning only \(dn\) decoder values while keeping the sparse-coding problem \((m,n,k)\) fixed.
The same structure reduces storage and turns the matching-pursuit correlation step \(\mathbf{W}^\top \mathbf{r}\) in OMP into an \(\mathcal{O}(dn)\) gather-and-reduce operation.
Our experiments show that across Pythia-70M/160M, Qwen2.5-3B, and Llama-3.2-1B residual-stream activations, varying \(d\) traces a consistent storage--fidelity frontier, and that at the most compressed modern-LM setting, Qwen2.5-3B with \(d=7\) uses \(293\times\) fewer learned decoder values than the full dense decoder while retaining \(84\%\) of dense CE-loss recovered.
Control experiments show that the improved storage--fidelity tradeoff is driven by sparse, diverse decoder support structure rather than by fewer learned decoder values, and that when sparse and dense decoders are compared at matched parameter count, part of the remaining gap comes from encoder amortisation.
On the theoretical side, we show that expansion and column flatness are sufficient for identifiability of noiseless \(k\)-sparse codes, and we derive complementary sufficient conditions under which OMP recovers the support exactly.
\end{abstract}

\section{Introduction}
\label{sec:intro}

Mechanistic interpretability aims to reverse-engineer the internal representations of a neural network by decomposing them into human-understandable components~\citep{olah2020zoom,elhage2022superposition}.
A central hypothesis is \emph{superposition} which proposes that neural networks can represent more features than they have dimensions by exploiting the sparsity of feature activations~\citep{elhage2022superposition}.
Sparse autoencoders (SAEs) have emerged as a central tool for studying this phenomenon by decomposing the internal activations of neural networks into sparse linear combinations of interpretable features~\citep{cunningham2023sparse, bricken2023monosemanticity, gao2024scaling, templeton2024scaling}.
SAEs are built by learning an {\em encoder} that maps an activation vector $\mathbf{h} \in \mathbb{R}^m$ to sparse feature coefficients $\mathbf{x} \in \mathbb{R}^n$ with $n > m$ and a {\em decoder} which maps $\mathbf{x}$ back to activation space $\hat{\mathbf{h}} \in \mathbb{R}^m$.
Recent work~\citep{klindt2025superposition, oneill2025amortisation} has emphasised that SAE inference is a sparse-recovery problem: given a trained decoder $\mathbf{W}_{\mathrm{dec}}\in\mathbb{R}^{m\times n}$ with $m<n$ and an activation $\mathbf{h}\approx \mathbf{W}_{\mathrm{dec}} \mathbf{x}$, recover a sparse latent code $\mathbf{x}\in\mathbb{R}^n$. This is the regime studied in compressed sensing~\citep{candes2005decoding, donoho2006compressed}, where recovery is possible only because the latent vector is sparse and the measurement matrix has a suitable geometry.
Indeed, prior work has shown that replacing the encoder with classic compressed-sensing algorithms like orthogonal matching pursuit (OMP) \citep{tropp2007omp} or gradient pursuit \citep{blumensath2008gradient} can close the ``amortisation gap'' and improve reconstruction~\citep{gdm2024ito, oneill2025amortisation, costa2025mpsae}.

A niche line of work within compressed sensing known as {\em combinatorial compressed sensing} studies the problem of sparse recovery when $\mathbf{W}$ is the adjacency matrix of a bipartite left $d$-regular expander graph~\citep{berinde2008combining, jafarpour2009efficient, indyk2010sparse, mendozasmith2015expander}.
When $d \ll m$, these matrices are binary, sparse, and highly structured so they are cheap to store and enable efficient recovery algorithms with provable guarantees~\citep{xu2007efficient, berinde2008ssmp, mendozasmith2015expander, mendozasmith2017robust}.
Inspired by this line of work, we propose \textbf{Expander SAEs}: tied-weight TopK sparse autoencoders whose encoder and decoder dictionaries are masked by the adjacency matrix of a left $d$-regular expander graph. The mask reduces the number of learned decoder values from $mn$ to $dn$ while keeping the sparse-coding problem $(m, n, k)$ fixed.
Moreover, this same structure also turns the correlation step $\mathbf{W}_{\mathrm{dec}}^\top \mathbf{r}$ in matching-pursuit decoders into an $\mathcal{O}(dn)$ sparse gather-and-reduce rather than an $\mathcal{O}(mn)$ dense product. We make the following contributions.

\begin{itemize}
    \item \textbf{Architecture.} We introduce Expander SAEs, a parameter-efficient architecture that lets us separate dictionary quality from amortised-inference quality.

    \item \textbf{Storage-fidelity frontier.} We show that across four open language models (Pythia-70M/160M, Qwen2.5-3B, Llama-3.2-1B), varying the sparsity parameter of the dictionary ($d$) traces a smooth trade-off between learned decoder values and reconstruction fidelity, and that on Qwen2.5-3B, Expander-SAE ($d{=}7$) recovers $84\%$ of the full dense cross-entropy loss while using $293\times$ fewer learned decoder values.

    \item \textbf{Parallel OMP.} We propose a parallel implementation of OMP that exploits the left $d$-regular column structure, yielding a Pareto frontier of operating points between fast amortised encoder inference and high-fidelity iterative decoding.

    \item \textbf{Identifiability theory.} We motivate the architecture by proving that if the fixed mask expands every $2k$-feature subset and the learned decoder columns remain sufficiently flat on their supports, then every noiseless $k$-sparse code has a unique $k$-sparse explanation. Under a separate, stronger cumulative-coherence condition, OMP is also certified to recover the support exactly.

\end{itemize}

\paragraph{Related work.}
Our work is adjacent to a number of lines of research.
In SAE architectures, Gated~\citep{rajamanoharan2024gated}, JumpReLU~\citep{rajamanoharan2024jumping}, TopK~\citep{gao2024scaling}, BatchTopK~\citep{bussmann2024batchtopk}, Switch~\citep{mudide2024switchsae}, and Meta-SAEs~\citep{bussmann2025meta} all modify the encoder or sparsity mechanism while keeping the decoder dense.
To our knowledge, Expander SAEs are the first SAE architecture to evaluate fixed decoder column sparsity as a storage/interpretability constraint.
On the compressed-sensing for SAEs side, \citet{gdm2024ito} first applied gradient pursuit to frozen dense decoders, \citet{oneill2025amortisation} formalised the amortisation gap, and \citet{costa2025mpsae} developed matching-pursuit variants as inference procedures.
Closest to our sparse construction, \citet{ba2018ds2p} studied deep sparse generative models with sparse-column measurement matrices.

\section{Background}
\label{sec:background}

A transformer processes tokens through $L$ layers.
At layer $\ell$ it maintains a \emph{residual stream} vector $\mathbf{h}^{(\ell)} \in \mathbb{R}^m$ that encodes everything the model has computed about the current token up to that layer.
Each layer updates the residual stream via attention ($\mathrm{Attn}$) and MLP sub-layers according to
\begin{align}
\mathbf{h}^{(\ell+1/2)} &= \mathbf{h}^{(\ell)} + \mathrm{Attn}^{(\ell)}(\mathbf{h}^{(\ell)}), \\
\mathbf{h}^{(\ell+1)}   &= \mathbf{h}^{(\ell+1/2)} + \mathrm{MLP}^{(\ell)}(\mathbf{h}^{(\ell+1/2)}).
\end{align}
The \emph{superposition hypothesis}~\citep{elhage2022superposition} posits that $\mathbf{h}^{(\ell)}$ encodes far more than $m$ interpretable concepts simultaneously.
Each concept corresponds to a direction $\mathbf{w}_j \in \mathbb{R}^m$ in activation space, and the activation is approximately a sparse linear combination of these directions:
\begin{equation}
    \label{eq:superposition}
    \mathbf{h}^{(\ell)} \;\approx\; \sum_{j=1}^{n} x_j\,\mathbf{w}_j + \boldsymbol{\varepsilon}
    \;=\; \mathbf{W}\mathbf{x} + \boldsymbol{\varepsilon},
\end{equation}
where $\mathbf{W} = [\mathbf{w}_1 \mid \cdots \mid \mathbf{w}_n] \in \mathbb{R}^{m \times n}$ is a \emph{feature dictionary} with $m < n$ that is learned from data, $\mathbf{x} \in \mathbb{R}^n$ is a sparse feature-activation vector and $\boldsymbol{\varepsilon} \in \mathbb{R}^m$ is noise.
The latent vector is assumed to satisfy $\|\mathbf{x}\|_0 = k \ll n$, so decoding $\mathbf{x}$ from $\mathbf{h}^{(\ell)}$ falls within the class of problems studied by compressed-sensing. 

Applied to a residual-stream activation $\mathbf{h} \in \mathbb{R}^m$, a sparse autoencoder (SAE) simultaneously learns a pair of dictionaries $(\mathbf{W}_{\mathrm{enc}}, \mathbf{W}_{\mathrm{dec}}) \in \mathbb{R}^{n \times m} \times \mathbb{R}^{m \times n}$ and a pair of biases $(\mathbf{b}_{\mathrm{enc}}, \mathbf{b}_{\mathrm{dec}}) \in \mathbb{R}^n \times \mathbb{R}^m$ by modelling the activation's reconstruction as
\begin{equation}
    \label{eq:encoder-decoder}
    \tag{SAE}
    \left\{
    \begin{aligned}
        \mathbf{x}        &= \sigma\!\left(\mathbf{W}_{\mathrm{enc}}(\mathbf{h} - \mathbf{b}_{\mathrm{dec}}) + \mathbf{b}_{\mathrm{enc}}\right), \\
        \hat{\mathbf{h}}  &= \mathbf{W}_{\mathrm{dec}}\,\mathbf{x} + \mathbf{b}_{\mathrm{dec}},
    \end{aligned}
    \right.
\end{equation}
where $\sigma$ is a sparsifying nonlinearity such as ReLU$+\ell_1$~\citep{cunningham2023sparse,bricken2023monosemanticity}, TopK~\citep{gao2024scaling}, or JumpReLU~\citep{rajamanoharan2024jumping}.
Different SAE variants impose sparsity in different ways. ReLU SAEs often use an $\ell_1$ penalty, while TopK SAEs impose a hard cardinality constraint by retaining only the largest $k$ pre-activations. In this paper, sparsity is enforced by TopK retaining the $k$ largest pre-activations and zeroing the rest, with no additional sparsity penalty in the training loss.
At inference, given a fresh residual-stream activation $\mathbf{h} \in \mathbb{R}^m$ and the parameters of a trained model \eqref{eq:encoder-decoder} we solve the sparse recovery problem
\begin{equation}
    \label{eq:decoding}
    \min_{\mathbf{x} \in \mathbb{R}^n}\;
    \bigl\lVert \left(\mathbf{h} - \mathbf{b}_{\mathrm{dec}}\right) - \mathbf{W}_{\mathrm{dec}} \mathbf{x} \bigr\rVert_2^2
    \quad \text{subject to} \quad \lVert \mathbf{x} \rVert_0 \le k.
\end{equation}
The problem~\eqref{eq:decoding} is NP-hard, but seminal compressed-sensing theory~\citep{candes2005decoding, donoho2006compressed} showed that $\mathbf{x}$ is unique and can be recovered exactly when every small subset of columns of $\mathbf{W}_{\mathrm{dec}}$ behaves {\em almost} like an isometry. Namely, when there exists a constant $\delta_k \in (0,1)$ such that
\begin{equation}
    \label{eq:rip2}
(1 - \delta_k)\|\mathbf{x}\|_2^2 \le \|\mathbf{W}\mathbf{x}\|_2^2 \le (1 + \delta_k)\|\mathbf{x}\|_2^2
\end{equation}
for all $\mathbf{x} \in \mathbb{R}^n$ such that $\|\mathbf{x}\|_0 \leq k$.
Equation \eqref{eq:rip2} is known as the {\em restricted isometry property} (RIP-2) and is the basis of compressed-sensing theory.

\subsection{Combinatorial compressed sensing}
\label{sec:background_expanders}

Combinatorial compressed-sensing studies the problem of recovering a $k$-sparse vector $\mathbf{x} \in \mathbb{R}^n$ from $\mathbf{h} = \mathbf{M} \mathbf{x}$ where $\mathbf{M} \in \{0,1\}^{m \times n}$ is the adjacency matrix of a $(k, \varepsilon, d)$-expander graph.
For completeness, we define this concept from first principles.

For an integer $n \in \mathbb{N}$ we let $[n] = \{1, \dots, n\}$.
A bipartite graph is a tuple $G = (L, R, E)$ with node set $L \cup R$ and edge set $E$ such that $L \cap R = \emptyset$ and $E \subset L \times R$.
Such a graph is called {\em imbalanced} if $|R| < |L|$ and is called {\em left $d$-regular} if each left vertex has exactly $d$ neighbours.
The adjacency matrix of such a graph can be represented as a binary matrix $\mathbf{M} \in \{0,1\}^{m \times n}$ such that $m < n$ and $\|\mathbf{M}_j\|_0 = d$ for all $j \in [n]$, where we use the notation $\mathbf{M}_j$ to refer to the $j$-th column of $\mathbf{M}$.
For $S\subseteq[n]$, we let $\Gamma(S):=\{i\in[m]: \exists j\in S \text{ such that } \mathbf{M}_{ij}=1\}$. In other words, the set $\Gamma(S) \subset [m]$ is the set of {\em neighbours} of the node set $S \subset [n]$.
The graph $G$ is a \emph{$(k, \varepsilon, d)$-expander} if it is bipartite, left $d$-regular and 
\begin{equation}
    |\Gamma(S)| \ge (1 - \varepsilon) d |S|
\end{equation}
for every $S \subseteq [n]$ with $|S| \le k$.
The parameter $\varepsilon \in (0,1)$ is the {\em expansion parameter} and quantifies how many collisions the nodes of $S$ generate in $[m]$.
That is, how often $i \in \Gamma(S)$ has more than one neighbour. 
Although certifying that a graph is a $(k, \varepsilon, d)$-expander is NP-hard, it has been shown that for appropriate scaling of $m$, $n$, $k$, and $d$ a random binary matrix with $d \ll m$ ones per column is a $(k, \varepsilon, d)$-expander with high probability \citep{bah2013vanishingly}. 
Moreover, it has been shown that when $\varepsilon < 1/2$, the adjacency matrix $\mathbf{M} \in \{0,1\}^{m \times n}$ of expander graphs satisfies a restricted isometry property in the $\ell_1$ norm (RIP-1)~\citep{berinde2008combining}:
\begin{equation}
    \label{eq:rip1}
    (1 - 2\varepsilon)\,d\,\|\mathbf{x}\|_1 \;\le\; \|\mathbf{M}\mathbf{x}\|_1 \;\le\; d\,\|\mathbf{x}\|_1
\end{equation}
for all $\mathbf{x} \in \mathbb{R}^n$ such that $\|\mathbf{x}\|_0 \leq k$.
RIP-1 unlocks a recovery toolkit with different computational and geometric guarantees from the usual dense RIP-2 setting, including greedy and peeling algorithms~\citep{berinde2008ssmp, jafarpour2009efficient, indyk2010sparse} and GPU-parallelisable iterative schemes~\citep{mendozasmith2015expander, mendozasmith2017robust}.
Applied to the mechanistic interpretability setting, expansion guarantees that any small set of features activates a correspondingly large set of distinct activation coordinates, limiting the ambiguity with which a sparse activation pattern can be inverted.

\section{Expander SAE Architecture}
\label{sec:architecture}

An Expander SAE with hyperparameters $(m, n, d, k)$ has learnable parameters $\mathbf{V} \in \mathbb{R}^{m \times n}$, $\mathbf{b}_{\mathrm{enc}} \in \mathbb{R}^n$, and $\mathbf{b}_{\mathrm{dec}} \in \mathbb{R}^m$ which instantiates \eqref{eq:encoder-decoder} by setting
\begin{equation}
    \label{eq:esae}
    \mathbf{W}_{\mathrm{dec}} = \left(\mathbf{V} \odot \mathbf{M}\right) \mathrm{diag}(\boldsymbol{\nu})^{-1},
    \qquad
    \mathbf{W}_{\mathrm{enc}} = \mathbf{W}_{\mathrm{dec}}^{\top}
\end{equation}
where $\mathbf{M} \in \{0,1\}^{m \times n}$ is a binary mask with $\|\mathbf{M}_j\|_0=d$ for all $j \in [n]$ sampled at initialisation and $\boldsymbol{\nu} \in \mathbb{R}^n$ normalises each column of $\mathbf{M} \odot \mathbf{V}$ to unit $\ell_2$ norm.
The forward pass instantiates \eqref{eq:encoder-decoder} with $\sigma(\mathbf{z}) = \mathrm{TopK}_k(\mathbf{z})$, which retains the $k$ entries with largest values, not largest magnitudes, and zeroes the rest.
Retained values can in principle be negative, but empirically, our trained checkpoints almost never select negative pre-activations.
Only the $dn$ nonzero positions of $\mathbf{V}$ receive gradients, giving $dn + n + m$ learnable parameters in total.

\subsection{Decoding with OMP}

A number of algorithms for combinatorial compressed-sensing exist but they are not very robust to noise.
We want to pick an algorithm that is able to exploit the geometry of expanders, but also achieve good reconstruction quality.
In Appendix \ref{app:omp} we compare several sparse-decoding algorithms and find that Orthogonal Matching Pursuit (OMP) \citep{tropp2007omp} gives the strongest reconstruction improvement among the tested decoders for Expander-SAE checkpoints. Algorithm \ref{alg:omp} gives the OMP variant used in our diagnostics.
\begin{algorithm}[h]
    \caption{OMP for Expander SAE}
    \label{alg:omp}
    \begin{algorithmic}[1]
        \REQUIRE measurement $\mathbf{y}$, decoder $\mathbf{W}$, bias $\mathbf{b}$, sparsity $k$
        \STATE $\mathbf{r} \leftarrow \mathbf{y} - \mathbf{b}$, $\quad S \leftarrow \emptyset$
        \FOR{$t = 1, \ldots, k$}
            \STATE $j^{*} \leftarrow \arg\max_{j \notin S} |\langle \mathbf{w}_j, \mathbf{r} \rangle|$
            \STATE $S \leftarrow S \cup \{j^{*}\}$
            \STATE $\hat{\mathbf{x}}_S \leftarrow \mathbf{W}[:, S]^{\dagger} (\mathbf{y} - \mathbf{b})$
            \STATE $\mathbf{r} \leftarrow \mathbf{y} - \mathbf{b} - \mathbf{W}[:, S]\,\hat{\mathbf{x}}_S$
        \ENDFOR
    \end{algorithmic}
\end{algorithm}

Although OMP is well studied, standard guarantees do not directly apply to learned, non-binary expander-style decoders.
The effect of non-binary weights is quantified from the column-flatness factor
\begin{equation}
\label{eq:decoder_flatness}
    \beta(\mathbf{W}_{\mathrm{dec}})
    :=
    \sqrt d
    \max_{i,j:\,\mathbf{M}_{ij}=1}
    |\mathbf{W}_{ij}|.
\end{equation}
Since each decoder column is supported on at most \(d\) entries and has unit \(\ell_2\)-norm, we have $1\le \beta(\mathbf{W}_{\mathrm{dec}})\le \sqrt d$.
Larger \(\beta\) means that a feature has concentrated more mass on fewer coordinates.
The resulting identifiability guarantee is stated in Theorem \ref{thm:expander_sae_recovery}

\begin{theorem}[Identifiability for learned SAE decoders]
\label{thm:expander_sae_recovery}
Let \(\mathbf{W}_{\mathrm{dec}}\in\mathbb{R}^{m\times n}\) be an Expander SAE decoder with unit-normalised columns supported on a left \(d\)-regular mask \(\mathbf{M}\).
Assume that \(\mathbf{M}\) is a \((2k,\varepsilon,d)\)-expander and let $\beta := \beta(\mathbf{W}_{\mathrm{dec}})$. 
If
\begin{equation}
\label{eq:beta_expansion_condition}
    2\beta^2\varepsilon < 1,
\end{equation}
then every \(2k\)-sparse vector \(\mathbf{u}\in\mathbb{R}^n\) satisfies
\begin{equation}
\label{eq:weighted_expander_lower_bound}
    \|\mathbf{W}_{\mathrm{dec}}\mathbf{u}\|_1
    \ge
    \sqrt d
    \left(
        \frac1\beta - 2\beta\varepsilon
    \right)
    \|\mathbf{u}\|_1 .
\end{equation}

Consequently, for every \(k\)-sparse latent vector \(\mathbf{x}_\star\in\mathbb{R}^n\), the noiseless Expander SAE reconstruction model $\mathbf{h} = \mathbf{b}_{\mathrm{dec}} + \mathbf{W}_{\mathrm{dec}}\mathbf{x}_\star$ has \(\mathbf{x}_\star\) as its unique \(k\)-sparse explanation.
Equivalently, the sparse decoder \eqref{eq:decoding} recovers \(\mathbf{x}_\star\) exactly in the noiseless case.
\end{theorem}

Theorem \ref{thm:expander_sae_recovery} is an idealised noiseless identifiability result for the sparse inverse problem defined by the learned decoder and is proved in Appendix \ref{app:expander_sae_theory}.
Specifically, it says that under expansion and flatness, the decoder has no nonzero \(2k\)-sparse null vector, so two distinct \(k\)-sparse codes cannot produce the same centred activation.
Note that when \(\beta=1\), \eqref{eq:beta_expansion_condition} reduces to the classical lossless-expander condition \(\varepsilon<\frac12\) \citep{jafarpour2009efficient} and that for learned non-binary decoders, the expansion deficit must be smaller.
Condition~\eqref{eq:beta_expansion_condition} is therefore the weighted Expander SAE analogue of the usual lossless-expander condition.
A stronger cumulative-coherence condition gives a conservative sufficient condition for standard OMP recovery.

\begin{corollary}[OMP recovery on Expander SAEs]
\label{cor:omp_main}
Let \(\mathbf{W}=\mathbf{W}_{\mathrm{dec}}\) have unit-normalised columns supported on a left
\(d\)-regular mask \(\mathbf{M}\), and let \(\beta=\beta(\mathbf{W})\). If \(\mathbf{M}\) is a
\((k{+}1,\varepsilon,d)\)-expander and
\[
    \beta^2\,\varepsilon\,(2k{+}1) < 1 ,
\]
then in the noiseless model \(\mathbf{h} = \mathbf{b}_{\mathrm{dec}} + \mathbf{W}\mathbf{x}_\star\),
OMP recovers the support of every \(k\)-sparse \(\mathbf{x}_\star\) in \(k\)
steps.
\end{corollary}

A proof of Corollary~\ref{cor:omp_main} is also given in Appendix~\ref{app:expander_sae_theory}.

Note that the conditions above are worst-case uniform recovery guarantees, so they should not be read as a claim that every configuration used in Section~\ref{sec:experiments} is certified.
In fact, certifying the parameters of an expander graph or the RIP conditions of a matrix is NP-hard, so we provide Theorem \ref{thm:expander_sae_recovery} and Corollary~\ref{cor:omp_main} to \emph{motivate} the fixed-support architecture and connect it to established theory rather than to certify the regime in which our experiments operate.

Algorithm~\ref{alg:omp} uses the absolute-value selection rule $\arg\max_{j \notin S} |\langle \mathbf{w}_j, \mathbf{r} \rangle|$, the standard formulation under which the cumulative-coherence guarantees of Corollary~\ref{cor:omp_main} and \citet{tropp2007omp} apply.
We point out that our reference implementation uses the signed variant $\arg\max_{j \notin S} \langle \mathbf{w}_j, \mathbf{r} \rangle$, mirroring the convention that $\mathrm{TopK}_k(\mathbf{z})$ retains the $k$ largest \emph{signed} pre-activations~\citep{gao2024scaling}, but the two coincide on the trained Expander-SAE dictionaries we evaluate (Appendix~\ref{app:novelty}).
Therefore, signed and absolute-value argmax select the same support on these dictionaries, and the theoretical guarantees apply to the procedure actually run.

As is standard in compressed sensing, these theorems motivate an asymptotic interpretation as $k$,$m$,$n \rightarrow \infty$ at fixed $\rho=k/m$ and $\delta=m/n$~\citep{donoho2009observed}.
However, it should be noted that our experiments choose $k$ for reconstruction and interpretability, not to activate the theorem's hypotheses, so these are worst-case sufficient conditions that should be read as architectural motivation rather than certification of every experimental setting.
Appendix~\ref{app:bench_geometry} reports empirical certificate ratios and discusses the asymptotic regimes in which the conditions are non-vacuous.

That being said, we also stress that the condition of the theorem is nevertheless non-vacuous.
Let
\[
    m=\delta n,
    \qquad
    k=\rho m,
\]
with fixed \(\delta<1\) and fixed \(\rho>0\).
Counting requires \(2(1-\varepsilon)d\rho\le 1\), because the neighbourhood of \(2k\) columns must fit inside \(m\) rows.
For any fixed flatness target \(\beta_0\), if \(d>2\beta_0^2\), then one can choose a constant \(1/d<\varepsilon<1/(2\beta_0^2)\).
For sufficiently small fixed \(\rho>0\), random left \(d\)-regular masks are \((2k,\varepsilon,d)\)-expanders with high probability, so the theorem applies at positive sparsity density provided training yields \(\beta(\mathbf{W})\le\beta_0\).

\begin{figure*}[h]
    \centering
    \includegraphics[width=0.49\linewidth]{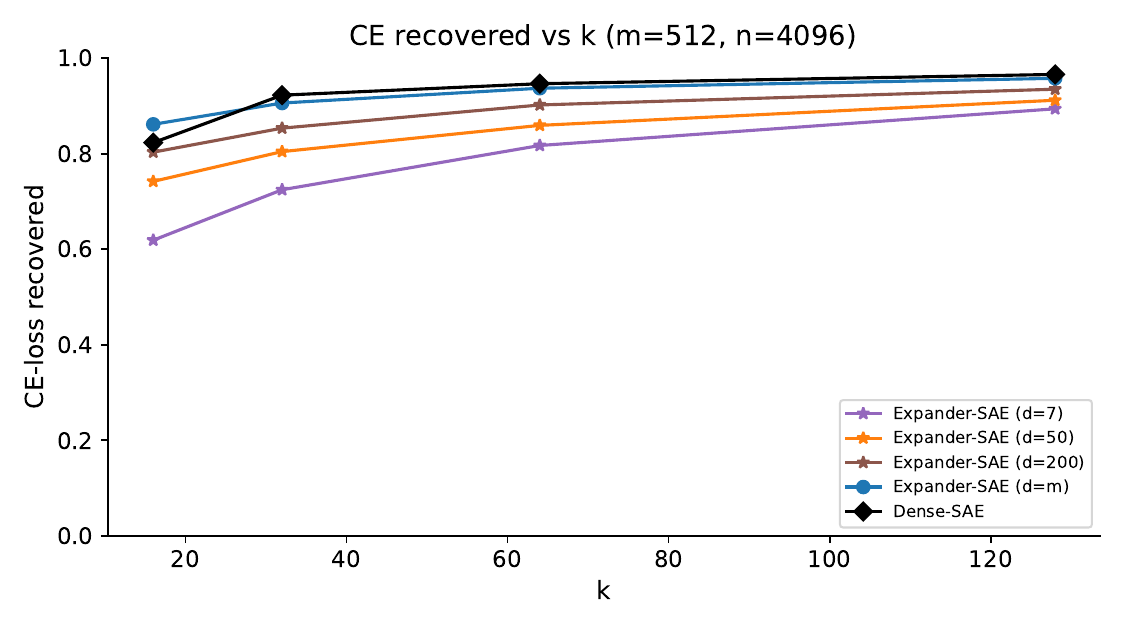}\hfill
    \includegraphics[width=0.49\linewidth]{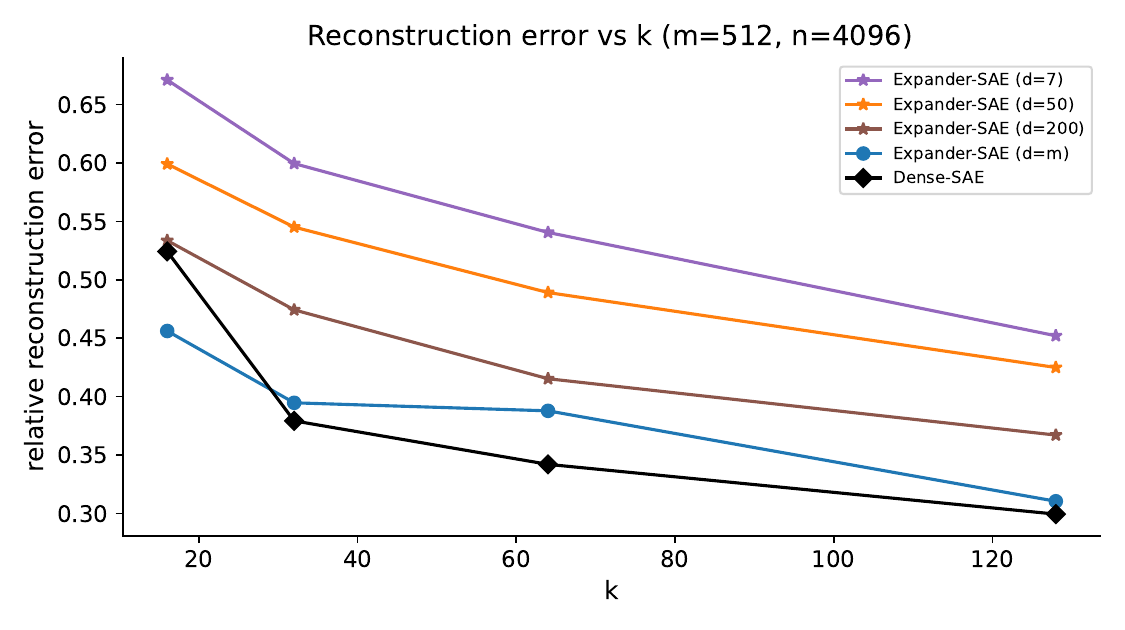}
    \caption{Storage--fidelity frontier on Pythia-70M layer 3 ($m{=}512$, $n{=}4096$, trained encoder). The left panel shows corrected CE-loss recovered versus TopK sparsity $k$, and the right panel shows relative reconstruction error versus $k$. Larger $d$ approaches the dense fidelity floor, while smaller $d$ retains high CE recovery with substantially fewer learned decoder values.}
    \label{fig:storage_frontier}
\end{figure*}

\section{Experiments}
\label{sec:experiments}

Before turning to the results, we summarise the experimental setup and
reporting conventions used throughout.\footnote{All training and
evaluation runs used A10G GPUs on Modal
(\url{https://www.modal.com}). The codebase was implemented with the
assistance of Claude Code (Anthropic Claude Opus 4.7) and is available
at \url{https://github.com/rodrgo/expander-sae}.}
We extract residual-stream activations from a fixed pre-trained language
model at a single hook site, train sparse autoencoders on those
activations, and report three metrics throughout the paper:
\begin{itemize}
    \item trained-encoder relative reconstruction error
    \[
    \mathrm{err}(\mathbf{h})
    =
    \frac{\lVert \mathbf{h} - \hat{\mathbf{h}} \rVert_2}
         {\lVert \mathbf{h} \rVert_2},
    \]
    which measures $\ell_2$ reconstruction quality on a held-out split of
    $5{,}000$ tokens;

    \item zero-ablation-normalised CE-loss recovered,
    \[
    \mathrm{CE\ rec}
    =
    \frac{\mathrm{CE}_{\mathrm{zero}} - \mathrm{CE}_{\mathrm{recon}}}
         {\mathrm{CE}_{\mathrm{zero}} - \mathrm{CE}_{\mathrm{clean}}},
    \]
    which measures downstream language-modelling behaviour;

    \item dead-feature fraction, the proportion of features that never
    fire on the held-out set.
\end{itemize}
We compare two main SAE families: \emph{Dense-SAE}, the standard sparse
autoencoder with independent encoder and decoder weights, and
\emph{Expander-SAE-$d$}, the tied-weight TopK sparse autoencoder of
Section~\ref{sec:architecture}, whose decoder columns are supported on a
frozen left-$d$-regular expander mask.
We also report the $d{=}m$ special case as a \emph{tied-dense} baseline,
which retains the tied architecture while removing the mask.
All architectures use the same TopK nonlinearity, unit-$\ell_2$ column
normalisation through the forward graph, and the optimiser,
learning-rate schedule, and dead-feature resampling rule described in
Appendix~\ref{app:training_details}.

Our headline setting uses residual-stream activations from layer 3 of
Pythia-70M~\citep{biderman2023pythia}, with $m{=}512$, $n{=}4096$, and
$k{=}64$, and a $200{,}000$-token training split together with a
$5{,}000$-token held-out split drawn from the
Pile~\citep{gao2020pile}.
We replicate the main trends on Pythia-160M~\citep{biderman2023pythia},
Qwen2.5-3B~\citep{qwen2025qwen25}, and
Llama-3.2-1B~\citep{grattafiori2024llama} at residual-stream hooks.
Unless explicitly marked as iterative OMP, all main-paper metrics use
the trained encoder.
We treat OMP and other compressed-sensing decoders as offline
diagnostics; the main matched-parameter OMP comparison appears in
Section~\ref{sec:exp_practical}, with full details in
Appendix~\ref{app:omp}.

\subsection{Varying $d$ traces a smooth storage--fidelity frontier across models}
\label{sec:exp_frontier}

We first ask whether varying $d$ traces a smooth storage--fidelity curve or merely picks out a single trade-off point.
Figure~\ref{fig:storage_frontier} plots both metrics against TopK sparsity $k$ for Expander-SAE at $d \in \{7, 50, 200\}$, the tied-dense baseline at $d{=}m$, and Dense-SAE on Pythia-70M layer~3.
At every $k$, increasing $d$ trades learned decoder values for fidelity along a smooth curve.
At $d{=}7$ the Expander decoder uses $73\times$ fewer learned values than the dense decoder while retaining $86\%$ of dense CE recovery ($0.817 / 0.947$), and at $d{=}200$ the gap to dense closes to within five percentage points of CE while still using $2.6\times$ fewer values.
CE-recovered is not strictly monotone in $d$ at every $k$, which is why we report it alongside relative reconstruction error rather than as the sole fidelity axis.

\begin{table*}[h]
    \centering
    \caption{Cross-model replication of the storage--fidelity frontier at $k{=}64$, three seeds. Storage is the ratio of learned decoder values to Dense-SAE at the same $(m,n)$; rel err and CE rec are held-out trained-encoder metrics. Starred/daggered cells denote seed-instability caveats detailed in Appendix~\ref{app:cross_model_full}.}
    \label{tab:cross_model}
    \small
    \setlength{\tabcolsep}{4pt}
    \begin{tabular}{llrrlrrr}
        \toprule
        Model & Layer & $m$ & $n$ & Method & Storage & rel err & CE rec \\
        \midrule
        \multirow{4}{*}{Pythia-70M}   & \multirow{4}{*}{$3$}  & \multirow{4}{*}{$512$}  & \multirow{4}{*}{$4{,}096$}    & Expander $d{=}7$    & $73\times$  & $0.541$        & $0.817$ \\
                                      &                       &                         &                                & Expander $d{=}50$   & $10\times$  & $0.489$        & $0.859$ \\
                                      &                       &                         &                                & Expander $d{=}200$  & $2.6\times$ & $0.415$        & $0.902$ \\
                                      &                       &                         &                                & Dense-SAE           & $1\times$   & $0.342$        & $0.947$ \\
        \midrule
        \multirow{4}{*}{Pythia-160M}  & \multirow{4}{*}{$8$}  & \multirow{4}{*}{$768$}  & \multirow{4}{*}{$6{,}144$}    & Expander $d{=}10$   & $77\times$  & $0.520$        & $0.767$ \\
                                      &                       &                         &                                & Expander $d{=}75$   & $10\times$  & $0.461$        & $0.836$ \\
                                      &                       &                         &                                & Expander $d{=}300$  & $2.6\times$ & $0.399$        & $0.889$ \\
                                      &                       &                         &                                & Dense-SAE           & $1\times$   & $0.332$        & $0.946$ \\
        \midrule
        \multirow{4}{*}{Qwen2.5-3B}   & \multirow{4}{*}{$12$} & \multirow{4}{*}{$2{,}048$} & \multirow{4}{*}{$16{,}384$} & Expander $d{=}7$    & $293\times$ & $0.764$        & $0.828$ \\
                                      &                       &                         &                                & Expander $d{=}30$   & $68\times$  & $0.703$        & $0.894$ \\
                                      &                       &                         &                                & Expander $d{=}102$  & $20\times$  & $0.681$        & $0.919$ \\
                                      &                       &                         &                                & Dense-SAE           & $1\times$   & $0.489$        & $0.983$ \\
        \midrule
        \multirow{4}{*}{Qwen2.5-3B}   & \multirow{4}{*}{$24$} & \multirow{4}{*}{$2{,}048$} & \multirow{4}{*}{$16{,}384$} & Expander $d{=}7$    & $293\times$ & $0.693$        & $0.875$ \\
                                      &                       &                         &                                & Expander $d{=}30$   & $68\times$  & $0.848^\star$  & $0.821^\star$ \\
                                      &                       &                         &                                & Expander $d{=}102$  & $20\times$  & $0.659$        & $0.884$ \\
                                      &                       &                         &                                & Dense-SAE           & $1\times$   & $0.591$        & $0.936$ \\
        \midrule
        \multirow{4}{*}{Llama-3.2-1B} & \multirow{4}{*}{$6$}  & \multirow{4}{*}{$2{,}048$} & \multirow{4}{*}{$16{,}384$} & Expander $d{=}7$    & $293\times$ & $0.809$        & $0.576^\dagger$ \\
                                      &                       &                         &                                & Expander $d{=}30$   & $68\times$  & $0.760$        & $0.743$ \\
                                      &                       &                         &                                & Expander $d{=}102$  & $20\times$  & $0.732$        & $0.816$ \\
                                      &                       &                         &                                & Dense-SAE           & $1\times$   & $0.568$        & $0.952$ \\
        \midrule
        \multirow{4}{*}{Llama-3.2-1B} & \multirow{4}{*}{$12$} & \multirow{4}{*}{$2{,}048$} & \multirow{4}{*}{$16{,}384$} & Expander $d{=}7$    & $293\times$ & $0.757$        & $0.608$ \\
                                      &                       &                         &                                & Expander $d{=}30$   & $68\times$  & $0.742$        & $0.707$ \\
                                      &                       &                         &                                & Expander $d{=}102$  & $20\times$  & $0.720$        & $0.758$ \\
                                      &                       &                         &                                & Dense-SAE           & $1\times$   & $0.499$        & $0.958$ \\
        \bottomrule
    \end{tabular}
\end{table*}

Table~\ref{tab:cross_model} shows that the frontier replicates across model families and scales.
On three open language-model families spanning $70$M to $3$B parameters and three different decoder designs (GPT-NeoX, Llama-style, and Qwen2-style), the same monotone $d \mapsto$ fidelity ordering holds at $k{=}64$.
The most extreme cell is Qwen2.5-3B at $d{=}7$, where Expander uses $293\times$ fewer learned values than the corresponding Dense-SAE ($114{,}688$ against $33{,}554{,}432$) and still recovers $84\%$ of its CE.
Pythia-70M layer 5 is omitted because Dense-SAE itself fails at that hook with $\mathrm{CE\ rec} < 0$ (see Table~\ref{tab:bench_cross_layer}), which we read as a model-stage limitation rather than an architecture-specific failure.

\subsection{Support geometry, not just sparsity or parameter count, drives the gain}
\label{sec:exp_controls}

We run three controls to disentangle the contribution of the expander mask from confounds with parameter count, column sparsity, and support geometry.

\paragraph{Matching parameters by shrinking $n$ changes the sparse-coding problem, not just the budget.} Figure~\ref{fig:controls} (bottom-right) trains a Dense-SAE at the same learned-parameter count by lowering $n$ to $n' = dn/m$. On this Pythia-70M layer 3, Expander beats it by $0.10$ to $0.21$ CE points at every shared budget, because reducing $n$ shifts the sparse-recovery ratio $\delta = m/n$, makes the dictionary undercomplete for small $d$, and at the smallest $d$ drops $n'$ below the null-space property $n > 2k$ where the cell is excluded entirely. At modern-LM scale the trained-encoder comparison appears to invert. Matched-Dense exceeds Expander by roughly $0.01$ CE on Qwen2.5-3B and by $0.060$ to $0.115$ CE on Llama-3.2-1B at both layers. This gap is largely an encoder amortisation effect rather than a decoder-quality difference. Iterative OMP applied symmetrically to both architectures collapses the matched-$n'{=}240$ against Expander $d{=}30$ gap from $+0.060$ to $+0.007$ CE on Llama and from $+0.012$ to $-0.005$ CE on Qwen, and narrows the matched-$n'{=}816$ against Expander $d{=}102$ gap on both models, as we unpack in Section~\ref{sec:exp_practical} and Table~\ref{tab:omp_at_scale}. The structural advantages of Expander supports therefore hold at every scale, namely the flat $(\texttt{values}, \texttt{rows})$ storage that admits structured-OMP decoding, the deterministic seed-regenerable masks, and the strong performance at the smallest $d$ where the matched-$n'$ cell falls below the NSP gate.

\paragraph{Sparse columns without support diversity leave many more features dead.} Figure~\ref{fig:controls} (top) keeps $(m, n, k, d)$ identical but partitions the rows into $G = \lfloor m/d \rfloor$ disjoint blocks and assigns every column to one block, deliberately destroying support diversity. It tracks Expander on relative reconstruction error at small and medium $d$, but at $d{=}200$ its dead-feature rate climbs by roughly $100\times$, from $0.7\%$ to $6.2\%$. Column sparsity alone therefore does not explain the parameter-efficiency frontier, and support diversity controls whether features remain in use at large $d$.

\paragraph{A dense model can be pruned into a good sparse decoder, but only with extra training.} Figure~\ref{fig:controls} (bottom-left) extracts a sparse mask from a pre-trained dense SAE by keeping the top-$d$ rows per column. It closes most of the relative-reconstruction-error gap to dense at every $d$, but requires a second training pass, so the support-extraction win is real but bought with $2\times$ the training compute.

\begin{figure*}[!tb]
    \centering
    \includegraphics[width=\linewidth]{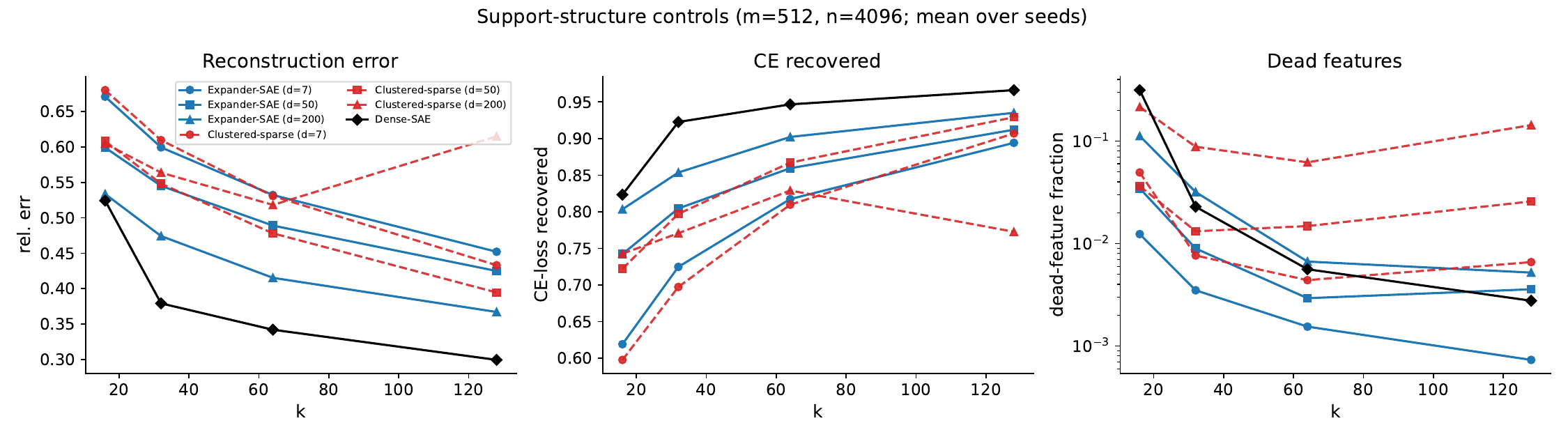}\\[0.5em]
    \includegraphics[width=0.6\linewidth]{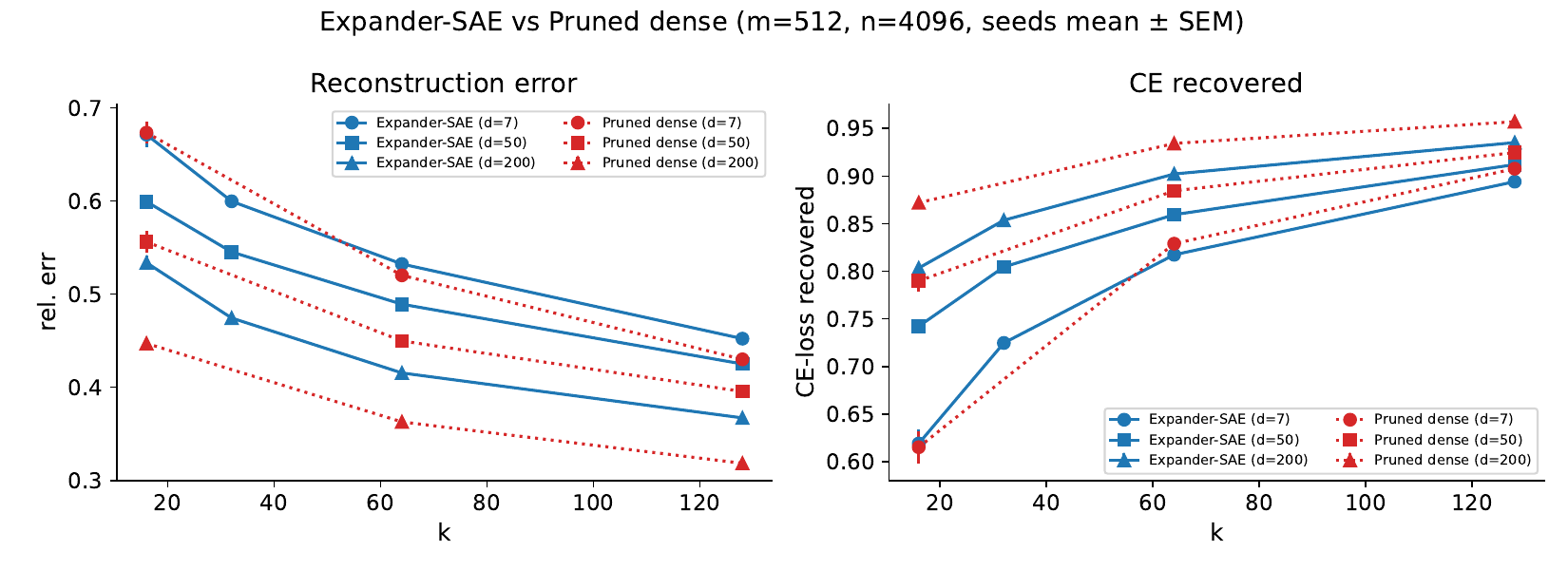}\hfill
    \includegraphics[width=0.38\linewidth]{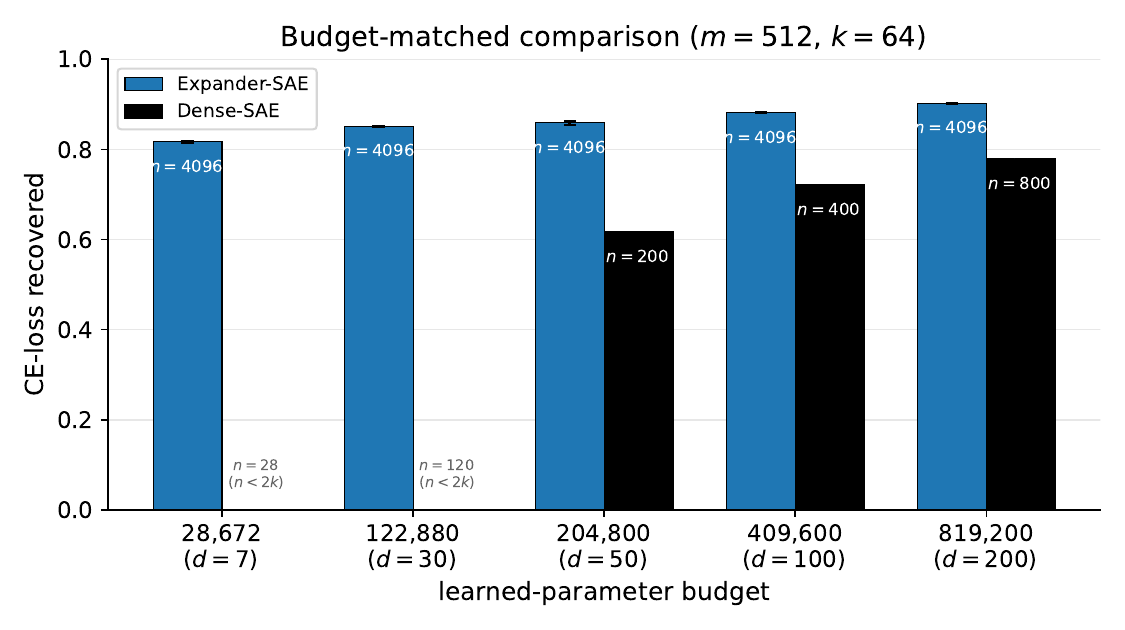}
    \caption{Three controls at $k{=}64$ on Pythia-70M layer 3, three seeds. The top row shows reconstruction error, CE-loss recovered, and dead-feature fraction (log scale) for Expander (solid) and Clustered-sparse (dashed). The bottom-left panel compares Expander (solid) against Pruned-retuned dense (dotted) at $d \in \{7, 50, 200\}$. The bottom-right panel reports the reduced-$n$ dense baseline at matched parameter budget.}
    \label{fig:controls}
\end{figure*}

\subsection{Most of the large-model gap is encoder amortisation, which iterative OMP largely closes}
\label{sec:exp_practical}

Table~\ref{tab:practical} consolidates the deployment-relevant numbers at $k{=}64$ on Pythia-70M layer 3, with on-disk storage, training cost, reconstruction error, CE recovery, and dead-feature fraction in one place. The remainder of this subsection unpacks how iterative OMP changes the picture at modern-LM scale and points to two further diagnostics in the appendix.

\begin{table}[h]
    \centering
    \caption{Deployment summary at $m{=}512$, $n{=}4096$, $k{=}64$ on Pythia-70M layer 3. Storage is total on-disk artefact size in KiB using the deployable layout; sparse models use flat $(\texttt{values},\texttt{rows})$ storage. Metrics are trained-encoder mean $\pm$ SEM over three seeds, except $\ddagger$ rows with two seeds. Each cell shows rel err over CE rec. Pruned dense counts both dense pretraining and sparse retuning. Full storage and metric breakdowns are in Tables~\ref{tab:bench_storage}--\ref{tab:bench_practical}.}
    \label{tab:practical}
    \footnotesize
    \setlength{\tabcolsep}{3pt}
    \renewcommand{\arraystretch}{1.1}
    \begin{tabular}{@{}llrccr@{}}
        \toprule
        Architecture & $d$ & \makecell{Stor.\\(KiB)} & \makecell{Tr.\\passes} & \makecell{rel err /\\CE rec.} & \makecell{Dead\\frac.} \\
        \midrule
        Dense-SAE     & $512$          & $16{,}402$     & $1$          & \makecell{$0.342{\pm}.000$\\$0.947{\pm}.001$} & $0.6\%$ \\
        \midrule
        \multirow{6}{*}{Expander-SAE}
                      & $7$            & $\mathbf{242}$ & $1$          & \makecell{$0.541{\pm}.006$\\$0.817{\pm}.001$} & $0.1\%$ \\
        \addlinespace[2pt]
                      & $50$           & $1{,}618$      & $1$          & \makecell{$0.489{\pm}.003$\\$0.859{\pm}.002$} & $0.3\%$ \\
        \addlinespace[2pt]
                      & $200$          & $6{,}418$      & $1$          & \makecell{$0.415{\pm}.001$\\$0.902{\pm}.001$} & $0.7\%$ \\
        \midrule
        Clust.-sparse & $200^\ddagger$ & $6{,}418$      & $1$          & \makecell{$0.518{\pm}.012$\\$0.829{\pm}.016$} & $\mathbf{6.2\%}$ \\
        \addlinespace[2pt]
        Pruned dense  & $200$          & $6{,}418$      & $\mathbf{2}$ & \makecell{$0.363{\pm}.002$\\$0.934{\pm}.001$} & $0.4\%$ \\
        \bottomrule
    \end{tabular}
\end{table}

\paragraph{Iterative OMP recovers the encoder amortisation gap at modern-LM scale.}
The trained-encoder CE-recovered numbers in Table~\ref{tab:cross_model} for Expander on Qwen2.5-3B and Llama-3.2-1B compress two effects into a single number, namely the quality of the learned $d$-regular decoder and the amortisation gap of the trained encoder against a stronger non-amortised sparse decoder on the same frozen dictionary. Replacing the encoder with iterative OMP at $L{=}1$ (top-$1$ pick by $\mathbf{W}_{\mathrm{dec}}^\top \mathbf{r}$, Cholesky refit on the active set, $k$ sequential iterations) decomposes the two effects, with results in Table~\ref{tab:omp_at_scale}. Two findings stand out. First, at the storage-extreme cell with $d{=}7$ and $293\times$ fewer learned values than full-Dense, iterative OMP gains $+0.073$ CE on Qwen and $+0.118$ CE on Llama, recovering most of the headroom to the larger-decoder cells and showing that the encoder amortisation gap is largest where storage compression is most aggressive. Second, the matched-parameter dense baseline gap from Section~\ref{sec:exp_controls} is mostly an encoder effect at $d{=}30$. Under iterative OMP applied symmetrically to both architectures, matched-Dense $n'{=}240$ against Expander $d{=}30$ shrinks from $+0.060$ to $+0.007$ CE on Llama and from $+0.012$ to $-0.005$ CE on Qwen, with Expander now slightly ahead. At the larger $d{=}102$ cell a smaller residual gap remains. The trained encoder remains the recommended online-deployment path at roughly $1.8$M tokens per second on a single A10G against iterative OMP's $13$k tokens per second (Table~\ref{tab:pareto_d7}), and iterative OMP is the right choice for offline analysis at the storage-extreme cell.

\begin{table}[h]
    \centering
    \caption{Encoder versus iterative OMP CE recovered at $k{=}64$ on Qwen2.5-3B and Llama-3.2-1B layer 12, three seeds. OMP runs on the same frozen checkpoints; Gain is Iter.\ OMP minus Encoder.}
    \label{tab:omp_at_scale}
    \small
    \setlength{\tabcolsep}{5pt}
    \begin{tabular}{lrrr}
        \toprule
        Architecture & \multicolumn{1}{c}{Encoder} & \multicolumn{1}{c}{Iter.\ OMP} & \multicolumn{1}{c}{Gain} \\
             & \multicolumn{1}{c}{CE rec}  & \multicolumn{1}{c}{CE rec}     & \\
        \midrule
        \multicolumn{4}{l}{\textbf{Qwen2.5-3B (layer 12)}} \\
        Expander $d{=}7$          & $0.828$ & $\mathbf{0.901}$ & $\mathbf{+0.073}$ \\
        Expander $d{=}30$         & $0.894$ & $0.914$          & $+0.020$          \\
        Expander $d{=}102$        & $0.919$ & $0.933$          & $+0.014$          \\
        Matched-Dense $n'{=}240$  & $0.906$ & $0.909$          & $+0.003$          \\
        Matched-Dense $n'{=}816$  & $0.945$ & $0.952$          & $+0.007$          \\
        \midrule
        \multicolumn{4}{l}{\textbf{Llama-3.2-1B (layer 12)}} \\
        Expander $d{=}7$          & $0.608$ & $\mathbf{0.726}$ & $\mathbf{+0.118}$ \\
        Expander $d{=}30$         & $0.707$ & $0.765$          & $+0.058$          \\
        Expander $d{=}102$        & $0.758$ & $0.798$          & $+0.040$          \\
        Matched-Dense $n'{=}240$  & $0.767$ & $0.772$          & $+0.005$          \\
        Matched-Dense $n'{=}816$  & $0.873$ & $0.880$          & $+0.007$          \\
        \bottomrule
    \end{tabular}
\end{table}

\paragraph{Features stay novel while coherence degrades gracefully}
Two qualitative diagnostics support the headline numbers. Activation-Jaccard novelty against a Dense-SAE reference is $81\%$ at $d{=}7$, well above the $68\%$ dense-vs-dense seed baseline and the firing-rate-matched null, and falls monotonically with $d$ to $5.8\%$ at $d{=}m$ (Appendix~\ref{app:novelty}). A blinded evaluation by two LLM judges (Claude Sonnet 4.5 and GPT-4o, $25$ features per architecture stratified by firing-rate quartile, anonymised IDs, three calls each at temperature $0$) rates Expander-$d{=}200$ at $3.72 \pm 0.19$ against Dense-SAE at $3.59 \pm 0.19$ on a $1$ to $5$ coherence scale, which is statistically indistinguishable. Expander-$d{=}7$ drops to $2.83 \pm 0.22$, indicating that interpretability degrades gracefully at the storage extreme. Inter-judge Spearman correlation on per-feature mean coherence is $\rho = 0.74$, with details in Appendix~\ref{app:llm_coherence}.

\section{Conclusion}
\label{sec:conclusion}

We introduced Expander SAEs, tied TopK sparse autoencoders whose decoder columns are supported on a fixed left-$d$-regular mask. The architecture keeps the sparse-coding dimensions $(m,n,k)$ fixed while reducing learned decoder values from $mn$ to $dn$. We prove a weighted-expander identifiability result for learned real-valued decoders under expansion and column-flatness assumptions, and empirically show across Pythia, Qwen, and Llama residual streams that varying $d$ traces a storage--fidelity frontier. At small budgets, iterative OMP recovers part of the encoder amortisation gap, while support-structure controls show that diverse sparse supports avoid clustered-mask dead-feature pathologies. We therefore view Expander SAEs not as a final replacement for dense SAEs, but as evidence that decoder support structure is an important and underexplored design axis. Promising next steps are larger-scale evaluations, low-rank corrections, learned sparse supports, and sharper data-dependent theory.

\bibliographystyle{plainnat}

\bibliography{refs}

\onecolumn
\appendix

\section{Training and evaluation details}
\label{app:training_details}

This appendix specifies the activation extraction, optimisation, and dead-feature resampling procedures used throughout the paper.

\paragraph{Activations.} Let $f_{\mathrm{LM}}^{(\ell)}(\,\cdot\,) \in \mathbb{R}^m$ denote the residual-stream output of layer $\ell$ of a fixed pre-trained language model. The headline configuration is Pythia-70M~\citep{biderman2023pythia} at $\ell = 3$ with $m = 512$, hooked on the residual-stream output of \texttt{lm.gpt\_neox.layers[3]} (equivalent to TransformerLens \texttt{blocks.3.hook\_resid\_post}). We tokenise sequences from the Pile~\citep{gao2020pile} with the model's tokeniser, truncate or pad to length $L = 128$, and run all $L$ tokens through the model under \texttt{no\_grad}. The activation dataset is the multiset of per-token residual-stream vectors,
\begin{equation}
    \label{eq:dataset}
    \mathcal{D} = \bigl\{\, \mathbf{h}_{s,t} := f_{\mathrm{LM}}^{(\ell)}\!\bigl(\mathbf{u}_s\bigr)_t \;\big|\; s = 1, \ldots, N_{\mathrm{seq}}, \; t = 1, \ldots, L \,\bigr\},
\end{equation}
with $|\mathcal{D}| = 210{,}000$, extracted in a single Modal A100 forward pass ($\sim\!40$\,s) and cached. Activations are used raw without whitening, and the decoder bias $\mathbf{b}_{\mathrm{dec}}$ is initialised to zero. All subsequent SAE training and evaluation runs use A10G GPUs on Modal.

\paragraph{Loss and sparsity.} The TopK nonlinearity retains the $k$ pre-activations with the largest values and zeroes the rest. The retained values are signed in principle, though the trained encoders we report produce non-negative codes (verified by the sign-aware diagnostic in Appendix~\ref{app:novelty}). A feature is counted as firing when its coefficient is nonzero. The training loss is per-sample $\ell_2$ reconstruction error with no sparsity penalty, since sparsity is enforced by TopK,
\begin{equation}
    \label{eq:loss}
    \mathcal{L} = \frac{1}{B} \sum_{b=1}^{B} \bigl\lVert \mathbf{h}_b - \hat{\mathbf{h}}_b \bigr\rVert_2^2,
\end{equation}
with batch size $B = 256$. The column normalisation in Eq.~\eqref{eq:esae} is part of the forward graph, so the loss is differentiable through it.

\paragraph{Optimisation.} We use Adam~\citep{kingma2015adam} with $\beta_1 = 0.9$, $\beta_2 = 0.999$, $\epsilon = 10^{-8}$, and a single-period cosine learning-rate schedule over $T = 5000$ steps,
\begin{equation}
    \label{eq:lr}
    \eta_t = 10^{-5} + \tfrac12 \bigl(3 \cdot 10^{-4} - 10^{-5}\bigr)\bigl(1 + \cos(\pi t / T)\bigr).
\end{equation}
Gradients are clipped to global $\ell_2$ norm $1.0$. Mini-batches are drawn with \texttt{shuffle=True, drop\_last=True}, and the loader cycles through $\mathcal{D}_{\mathrm{train}}$ until $T$ steps are completed.

\paragraph{Dead-feature resampling.} Following~\citet{bricken2023monosemanticity, gao2024scaling}, every $T_{\mathrm{r}} = \max(1000, T/5)$ steps we identify columns that fired on fewer than $5$ of the last $T_{\mathrm{r}} \cdot B$ samples. For each dead column $j$ we pick the sample in the current mini-batch with the largest residual $\mathbf{r}_b$, project it onto the column's mask support $\mathcal{S}_j = \{i : M_{ij} = 1\}$, and reset
\begin{equation}
    \label{eq:resample}
    \mathbf{V}_{\mathcal{S}_j, j} \leftarrow \frac{1}{\sqrt{d}} \frac{(\mathbf{r}_b)_{\mathcal{S}_j}}{\lVert (\mathbf{r}_b)_{\mathcal{S}_j} \rVert_2}, \qquad
    (\mathbf{W}_{\mathrm{enc}})_{j,:} \leftarrow (\mathbf{W}_{\mathrm{dec}})_{:,j}^{\top}, \qquad
    b_{\mathrm{enc}, j} \leftarrow 0.
\end{equation}
Resampling is disabled when more than $80\%$ of features are dead.

\section{OMP and inference-method diagnostics}
\label{app:omp}

This appendix benchmarks four sparse decoders on the trained Expander SAE checkpoints from Appendix~\ref{app:cross_model_full}, then describes the engineering steps that exploit the $d$-regular structure to make iterative OMP practical at scale. Since the ground-truth sparse code $\mathbf{x}$ is unobserved, we measure performance on the reconstruction $\hat{\mathbf{h}}$ via the relative reconstruction error
\begin{equation}
    \label{eq:relerr}
    \mathrm{err}(\mathbf{h}) \;=\; \frac{\lVert \mathbf{h} - \hat{\mathbf{h}} \rVert_2}{\lVert \mathbf{h} \rVert_2},
\end{equation}
averaged over a held-out batch, alongside CE-loss recovered (Appendix~\ref{app:ce_protocol}) as a downstream check that a small $\ell_2$ gap preserves language-modelling behaviour.

\subsection{Comparing sparse-decoding procedures}

We freeze the trained Expander SAE decoders ($m{=}512$, $n{=}4096$, $d \in \{7, 50, 200\}$) and evaluate four sparse-decoding procedures at the matching $k$. The trained encoder, OMP, NIHT~\citep{blumensath2010normalized}, and CoSaMP~\citep{needell2009cosamp} all solve Eq.~\eqref{eq:decoding} on the same held-out activations, and we report per-sample relative reconstruction error averaged over activations and seeds.

OMP reduces reconstruction error by $15$ to $25\%$ relative to the trained encoder across every $(d, k)$ tested, consistent with the amortisation-gap findings of \citet{oneill2025amortisation} and \citet{gdm2024ito}. NIHT and CoSaMP close part of the gap but do not match OMP, in contrast to their typical behaviour on idealised CS benchmarks~\citep{blanchard2011phase}. SAE activations are not exactly $k$-sparse in the learned dictionary, and OMP's greedy one-column-at-a-time support selection appears more robust to this mismatch than NIHT's hard-thresholded iterative projection or CoSaMP's support-pruning strategy. We therefore use OMP for the remaining diagnostics.

\begin{figure}[h]
    \centering
    \includegraphics[width=\linewidth]{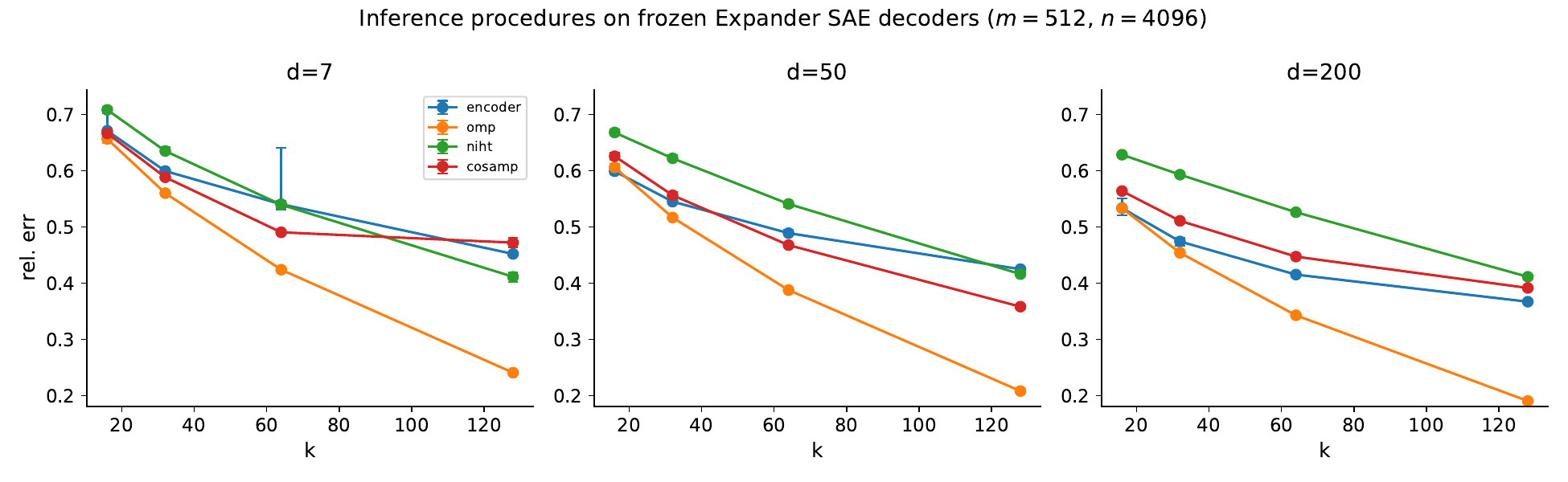}
    \caption{Relative reconstruction error versus $k$ for each decoding method on frozen Expander SAE decoders, shared $y$-range across $d$. Error bars show the seed min/max range.}
    \label{fig:omp_relerr}
\end{figure}

\subsection{Structured OMP for expander dictionaries}

The $d$-regular column structure admits three optimisations that turn iterative OMP from a CPU-bound diagnostic into a GPU-batched inference procedure, and a block-size sweep then exposes a continuum of operating points between iterative and single-shot decoding. Table~\ref{tab:enc_vs_omp} reports throughput and reconstruction error for every variant, and Table~\ref{tab:pareto_d7} sweeps the block size $L$ at $d{=}7$.

The $(\texttt{values}, \texttt{rows})$ storage already used by the encoder forward turns the correlation step $\mathbf{W}^\top \mathbf{r}$ into an $O(dn)$ gather and segment-sum rather than an $O(mn)$ dense GEMV, and turns the residual update into an $O(d|\mathrm{support}|)$ scatter. This yields an $11\times$ throughput improvement at $d{=}7$, $8\times$ at $d{=}50$, and $3\times$ at $d{=}200$ over vanilla OMP, a roughly $m/d$ scaling consistent with Amdahl's law applied to the correlation share of per-iteration cost. The Dense-SAE row still improves by $1.5\times$ because the residual-update scatter is cheaper than a dense $\mathbf{W}\hat{\mathbf{x}}$ matmul. Replacing the per-iteration \texttt{numpy.linalg.lstsq} with an incremental QR refit that grows the active-set factorisation via modified Gram-Schmidt reduces the lstsq cost from $O(mk^2 + k^3)$ to $O(mk + k^2)$ per iteration, and the combined structure-aware and incremental-QR implementation reaches $70\times$ over vanilla OMP at $d{=}7$ ($224$ tok/s on CPU), $19\times$ at $d{=}50$, and $5\times$ at $d{=}200$. Since the OMP iteration loop is sequential within a sample but independent across samples, casting Q, R, and the intermediate gather to bfloat16 (with the $k\times k$ triangular solve kept in fp32 for numerical safety) and batching the largest $B$ that fits in GPU memory reaches $11{,}587$ tokens/s at $d{=}7$, $2{,}245$ at $d{=}50$, $542$ at $d{=}200$, and $207$ for Dense-SAE on a single A10G.

Two further variants exploit the $d$-regular structure to amortise OMP's $k$ sequential iterations. Generalised OMP~\citep{wang2012gomp} picks the top-$L$ correlated columns per outer iteration, and on a $d$-regular dictionary a block of $L$ picks hits nearly disjoint rows whenever $Ld \le m$, keeping the QR refit well-conditioned. Cholesky refit replaces the per-block sequential picks with one batched Cholesky-on-normal-equations solve, forming the Gram matrix $\mathbf{W}_S^\top \mathbf{W}_S$ via two \texttt{bmm}s and producing $\mathbf{x}_S$ in three CUDA kernels per block. A Triton kernel computes $\mathbf{W}^\top \mathbf{r}$ directly from the $(\texttt{values}, \texttt{rows})$ storage, and at $d{=}7$ a further specialisation computes $\mathbf{W}_S^\top \mathbf{W}_S$ via a $(d \cdot d)$-unrolled inner loop over column-pair index intersections, never materialising the dense $\mathbf{W}_S$ tensor and recovering $\mathbf{W}_S^\top \mathbf{y}$ as a free gather at the picked support. Sweeping $L \in \{1, 2, 4, 8, 16, 32, 64\}$ at $d{=}7$, where $\lfloor m/d \rfloor = 73 \ge k$ so $L{=}64$ completes in a single block, traces the Pareto frontier in Table~\ref{tab:pareto_d7}. At $L{=}64$ single-shot decoding reaches $689$k tokens/s with rel err $0.546$, $0.024$ better than the encoder's $0.570$ on the same hardware at only $2.6\times$ throughput cost. At $L{=}4$ rel err drops to $0.465$ at $54$k tokens/s, within $0.007$ of full iterative OMP at $33\times$ encoder cost. The iterative limit $L{=}1$ delivers OMP's full $0.458$ rel err at $13.3$k tokens/s, $135\times$ slower than the encoder.

\begin{table}[!htbp]
    \centering
    \caption{Trained encoder versus OMP variants at $k{=}64$, seed 0, Pythia-70M layer 3. The $*$ marker denotes single-thread Apple silicon CPU and $\ddagger$ denotes a single Modal A10G with bfloat16 and $B{=}1024$, with $B{=}256$ where annotated since Dense-SAE OOMs at $B{=}1024$. The speed-ratio column compares each row to the same-hardware encoder. CPU implementations agree to $10^{-5}$, and the bfloat16 GPU implementation agrees with fp32 within sample-averaging noise.}
    \label{tab:enc_vs_omp}
    \small
    \begin{tabular}{lrrrr}
        \toprule
        Architecture & Inference & rel err & tokens/s & vs same-HW encoder \\
        \midrule
        \multirow{6}{*}{Expander-SAE ($d{=}7$)}   & trained encoder$^*$                       & $0.533$ & $52{,}280$    & --                       \\
                                                  & OMP vanilla$^*$                           & $0.452$ & $3.19$        & $16{,}388\times$ slower  \\
                                                  & OMP structured + QR$^*$                   & $0.452$ & $224.09$      & $233\times$ slower       \\
                                                  & trained encoder$^\ddagger$                & $0.570$ & $1{,}798{,}710$ & --                     \\
                                                  & OMP structured + QR$^\ddagger$            & $0.458$ & $11{,}587$ & $155\times$ slower \\
                                                  & Cholesky + Triton + struct.\ Gram, $L{=}64^{\dagger\ddagger}$ & $0.546$ & $\mathbf{689{,}364}$ & $\mathbf{2.6\times}$ slower \\
        \midrule
        \multirow{8}{*}{Expander-SAE ($d{=}50$)}  & trained encoder$^*$                       & $0.485$ & $42{,}532$    & --                       \\
                                                  & OMP vanilla$^*$                           & $0.390$ & $3.15$        & $13{,}502\times$ slower  \\
                                                  & OMP structured + QR$^*$                   & $0.390$ & $60.23$       & $706\times$ slower       \\
                                                  & trained encoder$^\ddagger$                & $0.509$ & $1{,}916{,}684$ & --                     \\
                                                  & OMP structured + QR$^\ddagger$            & $0.380$ & $2{,}245$ & $854\times$ slower \\
                                                  & gOMP $L{=}4^\ddagger$                     & $0.390$ & $7{,}367$  & $260\times$ slower \\
                                                  & gOMP $L{=}m/d^\ddagger$                   & $0.409$ & $13{,}270$ & $144\times$ slower \\
                                                  & Cholesky refit + Triton corr.$^\ddagger$  & $0.409$ & $\mathbf{48{,}784}$ & $\mathbf{39\times}$ slower \\
        \midrule
        \multirow{6}{*}{Expander-SAE ($d{=}200$)} & trained encoder$^*$                       & $0.415$ & $64{,}540$    & --                       \\
                                                  & OMP vanilla$^*$                           & $0.344$ & $3.12$        & $20{,}686\times$ slower  \\
                                                  & OMP structured + QR$^*$                   & $0.344$ & $16.11$       & $4{,}007\times$ slower   \\
                                                  & trained encoder$^\ddagger$                & $0.440$ & $1{,}930{,}422$ & --                     \\
                                                  & OMP structured + QR$^\ddagger$            & $0.331$ & $542$ & $3{,}556\times$ slower \\
                                                  & gOMP $L{=}m/d^\ddagger$                   & $0.334$ & $\mathbf{1{,}063}$ & $\mathbf{1{,}818\times}$ slower \\
        \midrule
        \multirow{5}{*}{Dense-SAE}                & trained encoder$^*$                       & $0.343$ & $67{,}303$    & --                       \\
                                                  & OMP vanilla$^*$                           & $0.270$ & $3.14$        & $21{,}434\times$ slower  \\
                                                  & OMP structured + QR$^*$                   & $0.270$ & $5.48$        & $12{,}282\times$ slower  \\
                                                  & trained encoder$^\ddagger$                & $0.365$ & $2{,}249{,}358$ & --                     \\
                                                  & OMP structured + QR$^{\ddagger\,(B{=}256)}$ & $0.279$ & $\mathbf{207}$ & $\mathbf{10{,}867\times}$ slower \\
        \bottomrule
    \end{tabular}\\
    \footnotesize
    $\dagger$ The block-size sweep for this row at $d{=}7$ traces a complete rel-err versus throughput Pareto curve, shown in Table~\ref{tab:pareto_d7}.
\end{table}

\begin{table}[!htbp]
    \centering
    \caption{Block-size $L$ sweep for the Cholesky-refit + Triton-correlation + structured-Gram OMP variant at $d{=}7$ ($m{=}512$, $n{=}4096$, $k{=}64$, A10G, bf16, $B{=}1024$, seed 0). $L$ controls the number of columns added to the active set per outer iteration, with $L{=}1$ recovering iterative OMP and $L{=}64$ collapsing to single-shot decoding. The final column compares each row to the same-hardware encoder throughput of $1{,}798{,}710$ tok/s. Supports are picked by \texttt{topk} on $\mathbf{W}^\top \mathbf{r}$, matching the encoder's TopK convention exactly.}
    \label{tab:pareto_d7}
    \small
    \begin{tabular}{rrrrrr}
        \toprule
        $L$ & outer iters & rel err & tokens/s & ms / 1024-batch & vs same-HW encoder \\
        \midrule
        $1$  & $64$ & $0.458$ & $13{,}305$  & $77.0$  & $135\times$ slower  \\
        $2$  & $32$ & $0.461$ & $26{,}605$  & $38.5$  & $68\times$ slower   \\
        $4$  & $16$ & $0.465$ & $53{,}924$  & $19.0$  & $33\times$ slower   \\
        $8$  & $8$  & $0.473$ & $107{,}529$ & $9.52$  & $17\times$ slower   \\
        $16$ & $4$  & $0.489$ & $214{,}774$ & $4.77$  & $8.4\times$ slower  \\
        $32$ & $2$  & $0.513$ & $410{,}614$ & $2.49$  & $4.4\times$ slower  \\
        $64$ & $1$  & $0.546$ & $\mathbf{689{,}364}$ & $\mathbf{1.49}$ & $\mathbf{2.6\times}$ slower \\
        \midrule
        \multicolumn{3}{r}{trained encoder (same HW)} & $1{,}798{,}710$ & $0.57$ & $1\times$ \\
        \bottomrule
    \end{tabular}
\end{table}

\subsection{The encoder as amortised OMP}

When $k \le m/d$, single-shot expander OMP and the trained encoder pick supports through the same structured $d$-regular gather followed by a \texttt{topk}, with the encoder additionally biased by its learned $\mathbf{b}_{\mathrm{enc}}$. Their supports overlap by roughly $82\%$ at $d{=}7$ on a held-out batch. Where the encoder emits the raw pre-activation, OMP solves a Cholesky on the picked support to recover the optimal least-squares coefficients, and the $0.024$ rel err improvement at $d{=}7$ is the amortisation gap of \citet{oneill2025amortisation} attributable to this coefficient refit at $2.6\times$ throughput cost. With sample-batching on GPU, OMP throughput is bounded by the $k$ sequential iterations rather than by dictionary size, so at $d{=}200$ and Dense-SAE the encoder-vs-OMP gap remains $\sim\!1{,}800\times$ and $\sim\!11{,}000\times$ where many sequential blocks dominate. The encoder remains the inference method used throughout the rest of the paper, and OMP is reported as an offline diagnostic that quantifies its amortisation gap.

\subsection{Geometry diagnostics: distance to identifiability and OMP certification}
\label{app:bench_geometry}

This appendix reports empirical estimates of the two quantities that govern the worst-case recovery certificates in Theorem~\ref{thm:expander_sae_recovery} and Corollary~\ref{cor:omp_coherence}, namely the expansion deficit of the fixed mask and the flatness factor $\beta(W)$ of the learned decoder, so that the distance between our finite trained models and the certified regimes is explicit. We define the ideal-decoder identifiability ratio
\[
    R_{\mathrm{id}}(k)
    \;:=\;
    2\beta_{\max}^2\,\widehat\varepsilon(2k),
\]
corresponding to the condition in Theorem~\ref{thm:expander_sae_recovery}, and the OMP sufficient-condition ratio
\[
    R_{\mathrm{OMP}}(k)
    \;:=\;
    \beta_{\max}^2\,\widehat\varepsilon(k+1)\,(2k+1),
\]
corresponding to Corollary~\ref{cor:omp_coherence}. Values below one certify the corresponding worst-case condition under the empirical expansion estimate, and values above one do not imply failure but mean the worst-case certificate is inactive. We compute both ratios at $m{=}512$, $n{=}4096$, seed 0. The empirical expansion deficit $\widehat{\varepsilon}_{\mathrm{greedy}}$ is obtained as a one-sided stress test, taking the maximum observed $1 - |\mathcal{N}(S)|/(d|S|)$ over size-$s$ candidate subsets produced by $10^3$ uniform random subsets, $30$ greedy collision-maximising restarts, and the top-overlap neighbourhood for every column. The flatness factor uses $\beta_{\max} = \max_j \sqrt{d}\,\|\mathbf{w}_j\|_\infty$ after column normalisation.

\begin{figure}[h]
\centering
\includegraphics[width=\linewidth]{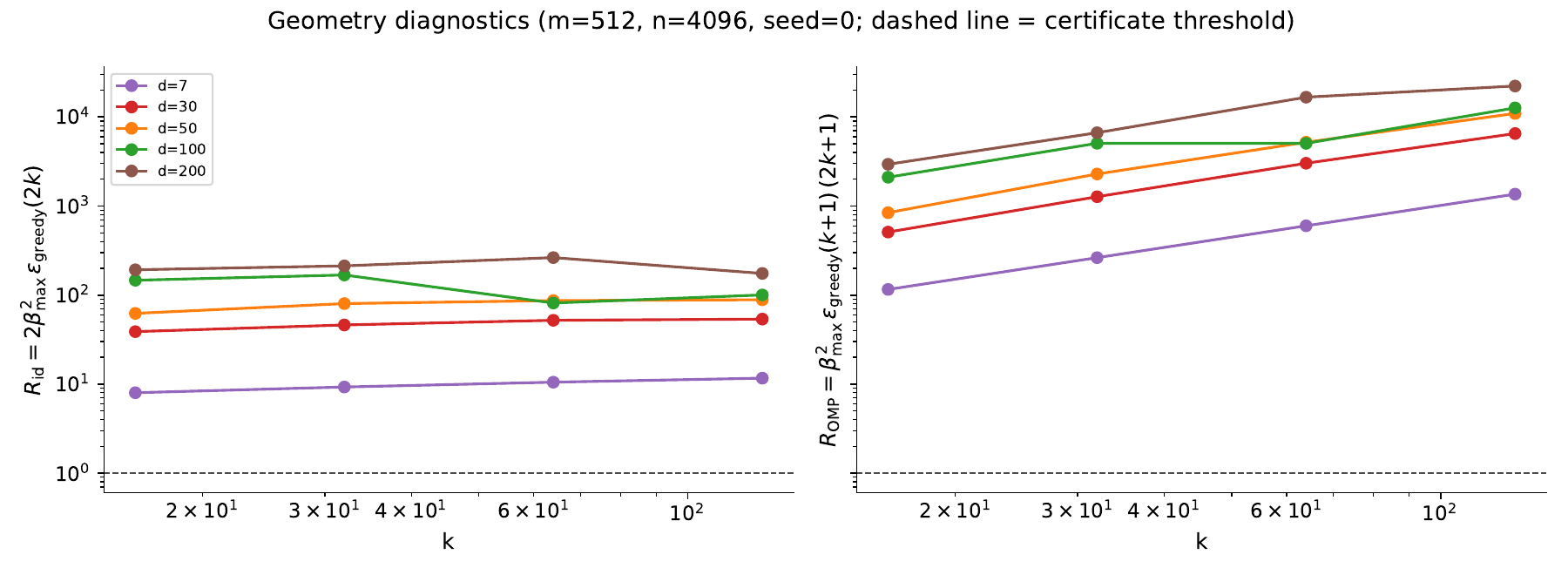}
\caption{Empirical certificate ratios on the trained Expander-SAE grid, one line per $d$. The dashed horizontal line marks the certificate threshold, and values below the line satisfy the worst-case sufficient condition.}
\label{fig:bench_geometry}
\end{figure}

Both ratios sit well above the certificate threshold across the entire trained grid, while OMP and the trained encoder reconstruct successfully on the same models. The current sufficient certificates are therefore loose in the operating regime probed by our language-model experiments, which is consistent with the asymptotic phase-transition framing in Section~\ref{sec:architecture} but does not by itself guarantee recovery on any individual cell.

\section{Active-support collision diagnostic}
\label{app:active_support_collisions}

This appendix tests the expansion property on the supports the trained SAE actually uses, rather than on worst-case random subsets. For each held-out activation we read off the active TopK support $S\subseteq[n]$ ($|S|=k$) and compute the empirical expansion deficit $1 - |\Gamma(S)|/(d|S|)$ and the duplicate-edge count $\sum_i \max(0, \deg_i(S) - 1)$ on the trained mask, giving a data-dependent counterpart of the worst-case stress test in Appendix~\ref{app:bench_geometry}.

\begin{figure}[h]
    \centering
    \includegraphics[width=\linewidth]{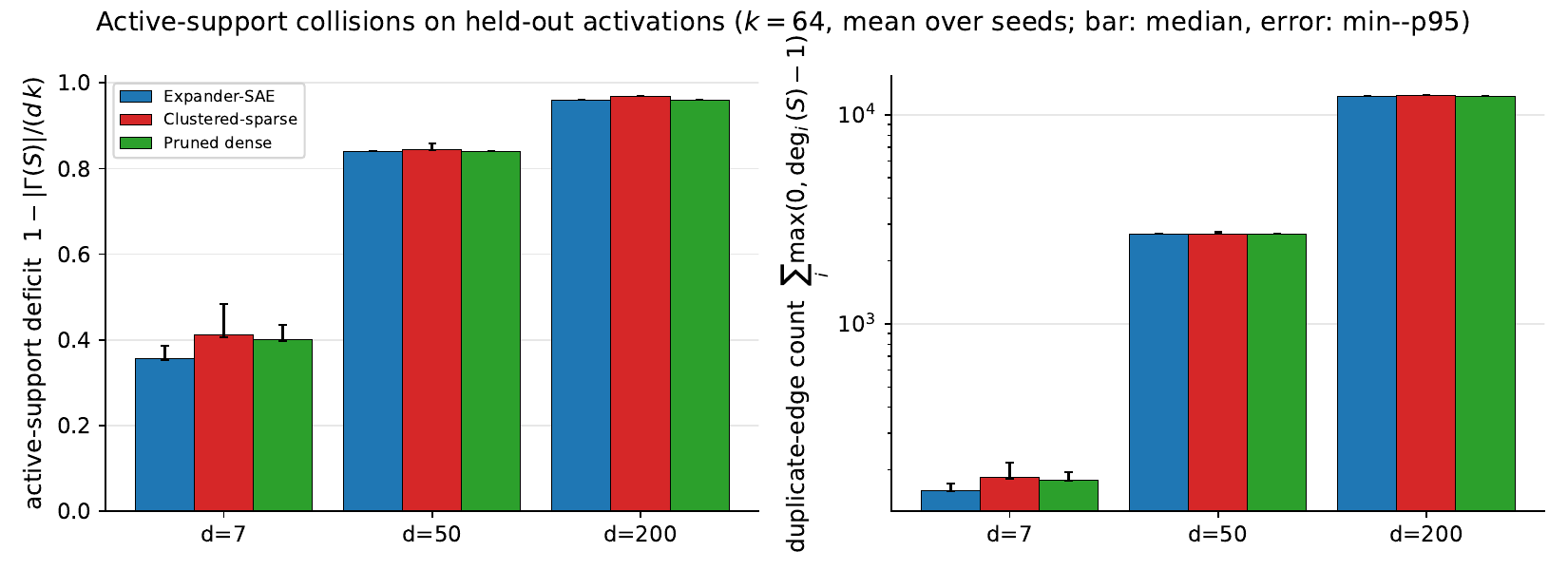}
    \caption{Active-support collision diagnostic at $k{=}64$ on held-out activations, three seeds. The left panel shows per-sample empirical expansion deficit on the active TopK support, with bars at the median and error bars spanning min to p95. The right panel shows per-sample duplicate-edge count on a log scale.}
    \label{fig:active_support_collisions}
\end{figure}

On the trained Expander-SAE's active supports, the collision deficit and duplicate-edge count fall below the matched Clustered-sparse control at every $d$. The gap is largest at $d{=}7$, where Expander has median deficit $0.36$ against Clustered's $0.41$. At $d \in \{50, 200\}$ both architectures saturate the row budget ($d|S| > m$), so the deficit is dominated by an unavoidable combinatorial floor and the two architectures look similar on this metric. The dead-feature pathology at $d{=}200$ in Figure~\ref{fig:support_structure} is therefore driven by the loss of distinct supports rather than by per-sample collisions, since Clustered $d{=}200$ admits only two possible column supports.

\section{Feature novelty diagnostics}
\label{app:novelty}

This appendix measures how different the Expander SAE's feature decomposition is from the Dense-SAE baseline using activation-pattern overlap on held-out Pythia-70M layer-3 activations. For each pair of features we compute the Jaccard overlap of their binary firing patterns under the convention that feature $j$ fires when $x_j \neq 0$, and call an Expander feature \emph{activation-novel} if its best-match Jaccard against the dense reference set is below $0.1$. We compare against an ensemble of dense SAEs rather than a single seed, since novelty relative to one seed can reflect ordinary seed variation. We use activation Jaccard rather than decoder-column cosine similarity because cosine is geometrically biased at low $d$, where a $d$-sparse decoder column has limited cosine to a dense column with mass spread evenly across coordinates regardless of whether the two represent the same feature. Decoder-cosine results appear in Appendix~\ref{app:decoder_cosine_novelty}, and at $d{=}m$ the two metrics agree. We calibrate the raw novelty fraction against two nulls. The first is a firing-rate-matched null in which $20$ random firing sets of the same size as each feature are sampled from the held-out activations, isolating the contribution of feature rarity to low Jaccard scores. The second is a dense-vs-dense baseline of $0.68$, computed between two dense tied SAEs trained with different seeds, which gives the rate at which one dense SAE has features the other does not match under the same threshold.

\begin{figure}[h]
    \centering
    \begin{minipage}[c]{0.49\linewidth}
        \centering
        \includegraphics[width=\linewidth]{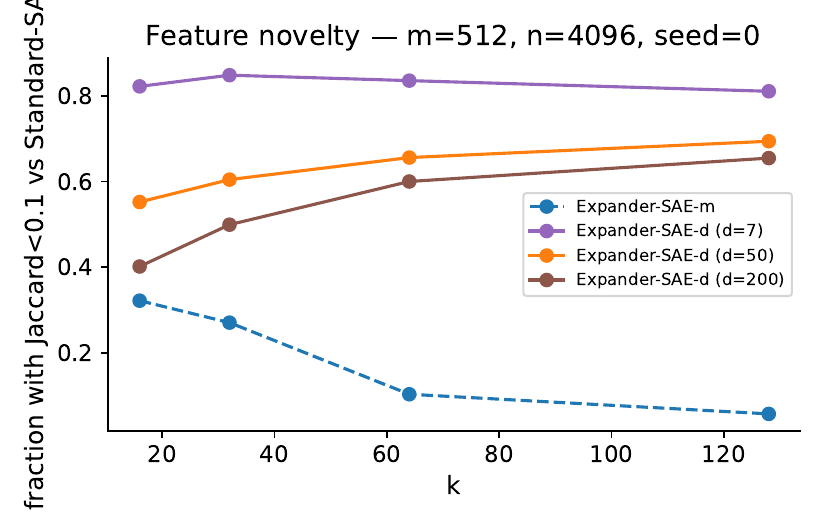}
    \end{minipage}\hfill
    \begin{minipage}[c]{0.49\linewidth}
        \centering
        \includegraphics[width=\linewidth]{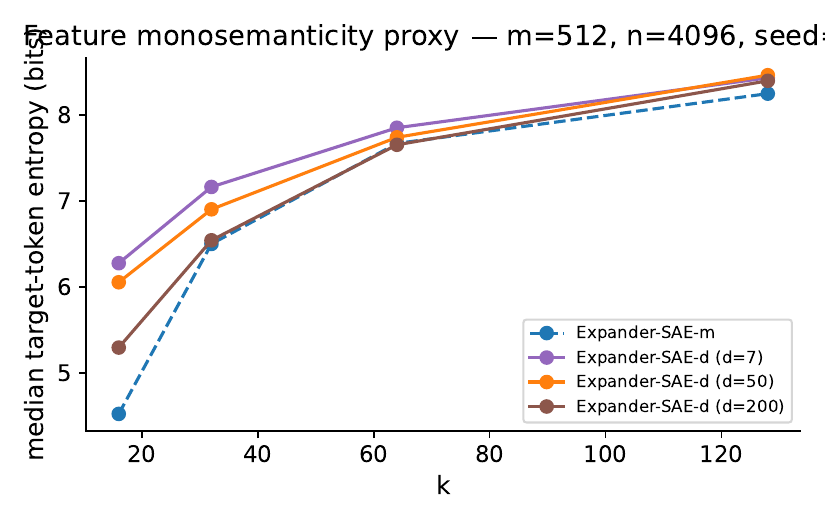}
    \end{minipage}

    \vspace{0.5em}

    \includegraphics[width=\linewidth]{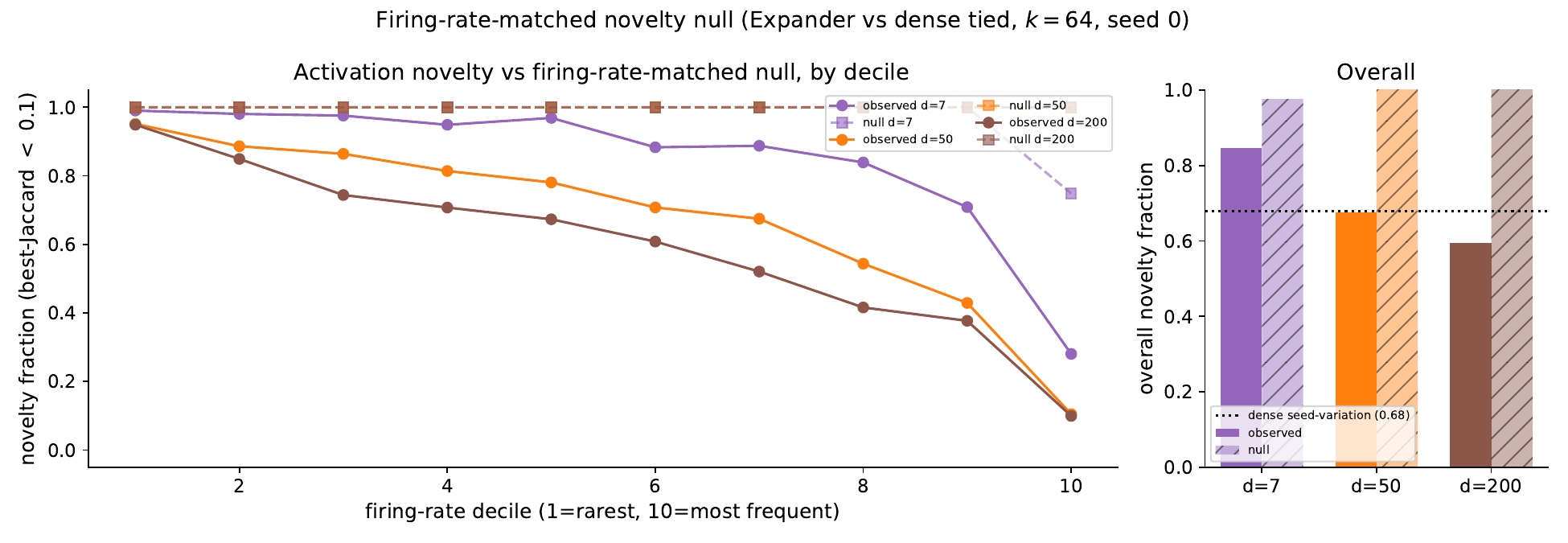}
    \caption{Feature novelty diagnostics for Expander-SAE against the Dense-SAE reference (trained encoder). The top-left panel shows activation-Jaccard novelty, the fraction of Expander features whose best-match Jaccard against the dense reference falls below $0.1$. The top-right panel shows median per-feature target-token entropy as a monosemanticity proxy. The bottom panels show the firing-rate-matched novelty null at $k{=}64$, seed 0, with novelty stratified by firing-rate decile on the left and the overall fractions per $d$ on the right.}
    \label{fig:novelty}
\end{figure}

The matched null sits close to one across all $d$ tested while the observed Expander novelty fraction is strictly lower, with $0.85$ against $0.97$ at $d{=}7$ and $0.60$ against $1.00$ at $d{=}200$. The raw novelty proxy is therefore partly inflated by feature rarity, since random sparse firing sets routinely fall below the Jaccard threshold, but the Expander features still overlap with dense features above the rarity baseline. Reading these against the dense-vs-dense seed-variation baseline of $0.68$ sharpens the picture further. At $d{=}7$ the observed Expander-vs-dense fraction of $0.85$ exceeds dense seed variation by about $0.17$, so the low-$d$ Expander lands on a decomposition more different from dense than two dense SAE seeds are from each other. At $d{=}50$ the observed fraction $0.68$ matches the dense-vs-dense baseline, indicating the difference is largely explained by seed variation. At $d{=}200$ the Expander fraction $0.60$ falls slightly below the baseline, consistent with the architecture converging towards the dense decomposition as $d \to m$. Low-$d$ Expander SAEs therefore land on a substantively different decomposition from dense baselines, while at higher $d$ the difference falls within the seed-variation noise floor.

\paragraph{Sign-aware firing convention.} We recompute the activation-Jaccard novelty under three firing definitions, $x_j \neq 0$, $x_j > 0$, and $x_j < 0$. For all trained Expander-SAE and Dense-SAE checkpoints in our sweep, the $x_j < 0$ rule yields zero alive features, since the TopK-of-largest-pre-activations operator with $n \gg k$ drives the trained encoders to non-negative codes in practice. The $x_j > 0$ rule therefore coincides exactly with $x_j \neq 0$, and the analysis above is unchanged under sign-aware splitting.

\subsection{Feature reliability diagnostics}
\label{app:reliability}

This appendix reports firing-rate and target-token entropy summaries for the activation-novel and shared feature subsets that the novelty analysis in Appendix~\ref{app:novelty} produces. We classify a feature as \emph{activation-novel} when its best-match Jaccard against the dense reference falls below $0.1$, and as \emph{shared} when the best-match Jaccard exceeds $0.5$, computed at $k{=}64$ and seed 0 over $128{,}000$ held-out Pythia-70M layer-3 tokens. Table~\ref{tab:appD_feature_stats} reports the median firing rate and median target-token entropy within each subset across $d$.

\begin{table}[h]
    \centering
    \caption{Per-feature firing rate and target-token entropy, taken as the median over the indicated subset, for Expander-SAE at five $d$ values against the dense tied reference at $k{=}64$, seed 0. Counts are out of $n{=}4096$ total features, and \emph{alive} means firing on at least one held-out token.}
    \label{tab:appD_feature_stats}
    \small
    \begin{tabular}{rrrrrrrr}
        \toprule
        $d$ & alive & novel ($J{<}0.1$) & shared ($J{>}0.5$) & \shortstack{rate med\\novel} & \shortstack{rate med\\shared} & \shortstack{entropy med\\novel (bits)} & \shortstack{entropy med\\shared (bits)} \\
        \midrule
          7 & $4096$ & $3424$ & $22$ & $0.82\%$ & $2.51\%$ & $7.90$ & $7.72$ \\
         30 & $4096$ & $2903$ & $20$ & $0.96\%$ & $3.18\%$ & $7.97$ & $7.90$ \\
         50 & $4096$ & $2688$ & $20$ & $0.86\%$ & $32.51\%$ & $7.87$ & $10.31$ \\
        100 & $4095$ & $2528$ & $34$ & $0.86\%$ & $3.15\%$ & $7.90$ & $7.06$ \\
        200 & $4096$ & $2459$ & $57$ & $0.84\%$ & $2.82\%$ & $7.77$ & $6.78$ \\
        \bottomrule
    \end{tabular}
\end{table}

Two qualitative observations follow. First, novel features fire on rarer subsets than shared features at every $d$, with median firing rate $0.8$ to $1.0\%$ against $2.5\%$ or higher, which is consistent with the firing-rate-matched null in Appendix~\ref{app:novelty} explaining a substantial part of the raw novelty fraction. Second, the median target-token entropy of novel features is roughly stable across $d$ at $7.8$ to $8.0$ bits, indicating that novel features fire on a broad set of token identities rather than concentrating on a small set, since the unconditional unigram entropy of the Pile under Pythia-70M's tokeniser sits in the low double digits. Shared features' entropies are more variable across $d$, reflecting the small per-bucket counts of $20$ to $57$ features at any $d$. Together these diagnostics show that activation-novel features are reproducible under our proxies in the sense that their firing sets are not driven by single rare tokens, while leaving open the stronger question of whether they are more interpretable than shared features. Split-half firing-rate reliability and token-frequency stability across halves are computed in \texttt{experiments/feature\_analysis.py} (\texttt{\_split\_half\_reliability}) and follow the same pattern.

\section{Blinded LLM evaluation of feature dashboards}
\label{app:llm_coherence}

This appendix tests whether the Expander SAE's distinct feature decomposition is interpretable, since novelty against the dense baseline (Appendix~\ref{app:novelty}) does not by itself establish meaningfulness. We sample $25$ features each from Expander-SAE at $d{=}7$, Expander-SAE at $d{=}200$, and the Dense-SAE baseline (Pythia-70M layer 3, $k{=}64$, seed 0), stratified across firing-rate quartiles per architecture so that quartile distributions match. For each feature we render a dashboard of the $15$ top-activating Pile contexts in $16$-token windows with the target token highlighted, under a global anonymised feature ID (\texttt{F01}--\texttt{F75}) that hides $d$, the decoder column, firing rate, and architecture name from the judge. Two LLM judges, Claude Sonnet 4.5 and GPT-4o, score each dashboard on a $1$ to $5$ coherence scale where $5$ marks a clear unifying theme and $1$ marks unrelated activations, and propose a $1$ to $3$-word concept label or ``no clear concept'', at temperature $0$ with three calls per (feature, judge) pair. Inter-judge Spearman correlation on per-feature mean coherence is $\rho = 0.74$.

\begin{figure}[h]
    \centering
    \includegraphics[width=0.95\linewidth]{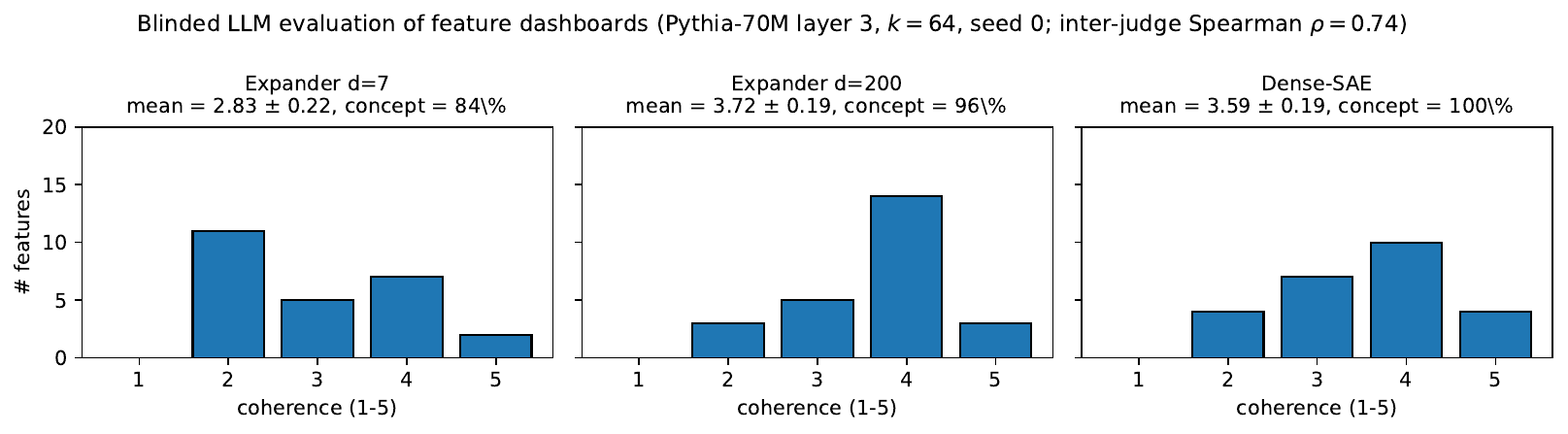}
    \caption{Per-architecture distribution of mean coherence ratings under blinded LLM evaluation, averaged across two judges and three calls each then rounded to integers $1$ to $5$. $n{=}25$ features per architecture, stratified by firing-rate quartile so that quartile distributions match across architectures. Architecture is invisible to the judges, with features identified only by anonymised ID.}
    \label{fig:feature_coherence}
\end{figure}

Mean coherence is $3.59 \pm 0.19$ for Dense-SAE, $3.72 \pm 0.19$ for Expander $d{=}200$, and $2.83 \pm 0.22$ for Expander $d{=}7$, with 95\% bootstrap CIs of $[3.21, 3.96]$, $[3.34, 4.06]$, and $[2.43, 3.27]$ respectively. The fraction of features assigned a non-empty concept label is $100\%$ for Dense-SAE, $96\%$ for Expander $d{=}200$, and $84\%$ for Expander $d{=}7$. At $d{=}200$, interpretability is statistically indistinguishable from the dense baseline while the Expander uses $2.6\times$ fewer learned values. At the storage extreme $d{=}7$ ($73\times$ fewer learned values), coherence drops by about $0.8$ points on the $1$ to $5$ scale, but $84\%$ of features still receive a concrete concept label, matching the activation-Jaccard novelty result that low-$d$ Expander finds a substantively different but still meaningful decomposition. The architectural constraint at extreme storage compression carries a measurable interpretability cost that scales with the storage savings, while at $d{=}200$ the cost falls below the measurement noise floor.

\section{Cross-model and cross-layer replication: full results}
\label{app:cross_model_full}

This appendix reports the full storage-fidelity sweep that Section~\ref{sec:exp_frontier} summarises, including all $(d, k)$ cells on the Pythia-70M headline configuration and the cross-layer, cross-model, and cross-architecture replications referenced in Table~\ref{tab:cross_model}. CE-loss recovered is computed by replacing the layer-$\ell$ residual stream of the language model with the SAE reconstruction and measuring next-token cross-entropy on held-out sequences,
\begin{equation}
    \label{eq:ce-recovered}
    \mathrm{CE\ recovered}
    \;=\;
    \frac{
        \mathrm{CE}_{\mathrm{zero}}-\mathrm{CE}_{\mathrm{recon}}
    }{
        \mathrm{CE}_{\mathrm{zero}}-\mathrm{CE}_{\mathrm{clean}}
    },
\end{equation}
where $\mathrm{CE}_{\mathrm{clean}}$ is the cross-entropy of the unmodified language model, $\mathrm{CE}_{\mathrm{zero}}$ replaces the activation by the all-zero vector at the hook site, and $\mathrm{CE}_{\mathrm{recon}}$ replaces it by $\hat{\mathbf{h}} = \mathbf{W}_{\mathrm{dec}}\mathbf{x} + \mathbf{b}_{\mathrm{dec}}$. A value of $1$ means the reconstruction matches the clean model's CE, $0$ means it is no better than zero-ablation, and negative values mean it is worse than zero-ablation. In Figure~\ref{fig:storage_frontier} we sweep Expander SAE at $m{=}512$, $n{=}4096$ over $d \in \{7, 50, 200\}$ and TopK levels $k \in \{16, 32, 64, 128\}$, comparing against the tied dense SAE at $d{=}512$ and Dense-SAE, with each point averaged over three seeds. Relative to the dense tied baseline, Expander uses $73\times$ fewer learned decoder values at $d{=}7$, $10\times$ at $d{=}50$, and $2.6\times$ at $d{=}200$. Figure~\ref{fig:storage_frontier} also shows that across $d$, the Expander SAE traces a smooth decoder-sparsity and fidelity frontier where larger $d$ approaches the dense fidelity floor and smaller $d$ retains high CE-loss recovered with substantially fewer learned decoder values.

\subsection{Cross-layer and cross-model replication on Pythia}

This subsection tests whether the storage-fidelity ratio depends on the specific Pythia-70M layer 3 hook. We re-extract Pythia-70M residual-stream activations at layers $1$ and $5$ and retrain Expander-SAE at $d \in \{7, 50, 200\}$ and Dense-SAE on each layer at $k{=}64$, three seeds, matching the $200{,}000$-token train and $5{,}000$-token test split. We then repeat the experiment on Pythia-160M ($m{=}768$, $n{=}6144$) at layers $4$ and $8$ with $d \in \{10, 75, 300\}$, chosen to match the $d/m$ ratios used at $m{=}512$ for storage savings of approximately $77\times$, $10\times$, and $2.6\times$ relative to Expander $d{=}m$. Table~\ref{tab:bench_cross_layer} reports trained-encoder relative reconstruction error and CE-loss recovered.

\begin{table}[h]
    \centering
    \caption{Cross-layer and cross-model replication at $k{=}64$, three seeds (mean reported; SEM is below $0.01$ on every cell except where flagged in the body). \emph{rel err} is trained-encoder relative reconstruction error on a $5{,}000$-token held-out split, and \emph{CE rec} is CE-loss recovered against the same hook site. Per-layer (CE clean / CE zero-ablation): Pythia-70M layers 1/3/5 ($3.44/8.93$, $3.44/12.73$, $3.44/8.75$); Pythia-160M layers 4/8 ($2.96/10.96$, $2.96/11.96$).}
    \label{tab:bench_cross_layer}
    \small
    \setlength{\tabcolsep}{4pt}
    \begin{tabular}{llrrrrrr}
        \toprule
        & & \multicolumn{2}{c}{Expander $d$} & \multicolumn{2}{c}{Expander $d^\prime$} & \multicolumn{2}{c}{Dense-SAE} \\
        Model & Layer & rel err & CE rec & rel err & CE rec & rel err & CE rec \\
        \midrule
        \multicolumn{8}{l}{\emph{Pythia-70M} ($m{=}512$, $n{=}4096$, $d \in \{7, 50, 200\}$ left to right inside each Expander column-pair)} \\
        \midrule
        Pythia-70M  & 1 & \multicolumn{2}{l}{$0.570 / 0.738$} & \multicolumn{2}{l}{$0.516 / 0.819$ \quad $0.446 / 0.883$} & $0.368$ & $0.938$ \\
        Pythia-70M  & 3 & \multicolumn{2}{l}{$0.532 / 0.817$} & \multicolumn{2}{l}{$0.489 / 0.859$ \quad $0.415 / 0.902$} & $0.342$ & $0.947$ \\
        Pythia-70M  & 5 & \multicolumn{2}{l}{$0.375 / -0.347$} & \multicolumn{2}{l}{$0.341 / 0.036$ \quad $0.433 / -0.163$} & $0.607$ & $-0.550$ \\
        \midrule
        \multicolumn{8}{l}{\emph{Pythia-160M} ($m{=}768$, $n{=}6144$, $d \in \{10, 75, 300\}$ left to right inside each Expander column-pair)} \\
        \midrule
        Pythia-160M & 4 & \multicolumn{2}{l}{$0.643 / 0.742$} & \multicolumn{2}{l}{$0.588 / 0.836$ \quad $0.515 / 0.901$} & $0.439$ & $0.949$ \\
        Pythia-160M & 8 & \multicolumn{2}{l}{$0.520 / 0.767$} & \multicolumn{2}{l}{$0.461 / 0.836$ \quad $0.399 / 0.889$} & $0.332$ & $0.946$ \\
        \bottomrule
    \end{tabular}
\end{table}

The storage-fidelity ratio replicates across both the layer dimension and the model-size dimension. On Pythia-70M layers $1$ and $3$ and on both tested Pythia-160M layers, Expander-SAE recovers within $5$ to $20$ percentage points of CE of Dense-SAE while using $73\times$ ($d{=}7$ at $m{=}512$) or $77\times$ ($d{=}10$ at $m{=}768$) fewer learned decoder values, and the medium and large $d$ values continue to trace the same monotone frontier. Pythia-70M layer 5 is a regime in which all SAE architectures we tested fail to preserve LM behaviour, since Dense-SAE itself yields a negative CE-loss recovered. The residual-stream output of the last GPT-NeoX block is harder to reconstruct without breaking the unembedding than at earlier layers regardless of architecture, and we read this as a known limitation of SAE-style probes at late layers rather than a feature specific to Expander supports.

\subsection{Qwen2.5-3B cross-architecture replication}

The Pythia experiments above use a single LM family. To check whether the storage-fidelity claim transfers to a modern non-Pythia architecture at a scale where SAE storage is operationally painful, we replicate on Qwen2.5-3B, a 3B-parameter decoder-only LM with hidden size $m{=}2048$, $36$ transformer blocks, RoPE positional encoding, grouped-query attention, and SwiGLU MLPs, all of which depart substantially from Pythia's GPT-NeoX style. We extract residual-stream activations at layers $12$ and $24$ (third and two-thirds of depth, mirroring the Pythia-160M pattern) and train Expander-SAE at $d \in \{7, 30, 102\}$, where $d{=}7$ matches the headline Pythia-70M configuration and $d \in \{30, 102\}$ correspond to $d/m \in \{1.5\%, 5\%\}$. We compare against a Dense-SAE at $n{=}16{,}384$ ($8\times$ overcompleteness) and matched-parameter reduced-$n$ Dense-SAEs at $n' \in \{240, 816\}$ that match Expander's unique-parameter count at $d \in \{30, 102\}$. The matched cell at $d{=}7$ would require $n'{=}56 < 2k$ and is excluded by the NSP gate. Every SAE is trained at $k{=}64$ on three seeds with the same $200$k-token / $5$k-token train/test split as the Pythia experiments. The mask sampler in \texttt{models/expander.py} re-rolls whenever any row receives no incoming edges, which is a low-probability event at these dimensions ($d \cdot n / m \approx 56$ at $d{=}7$).

\begin{table}[h]
    \centering
    \caption{Qwen2.5-3B replication at $m{=}2048$, $n{=}16{,}384$, $k{=}64$, three seeds (mean reported; standard deviation across seeds is below $0.02$ on every cell except where flagged with $^\star$). \emph{rel err} is trained-encoder relative reconstruction error on the $5{,}000$-token held-out split, and \emph{CE rec} is CE-loss recovered against the same hook site. The matched-parameter Dense-SAE columns ($d{=}30 \to n'{=}240$, $d{=}102 \to n'{=}816$) report Dense-SAEs trained at reduced $n$ to match Expander's unique-parameter count, and the $d{=}7$ matched cell is excluded because $n'{=}56 < 2k$. CE clean / CE zero-ablation are $2.29/15.98$ at layer 12 and $2.29/23.30$ at layer 24. The starred cell ($d{=}30$ at layer 24) has one of three seeds diverge during training (rel err $1.13$, included in the reported mean); without that seed the rel err mean is $0.704$ and CE rec is $0.854$.}
    \label{tab:bench_qwen2_5_3b}
    \small
    \setlength{\tabcolsep}{4pt}
    \begin{tabular}{lrrrrrrr}
        \toprule
        & & \multicolumn{3}{c}{Expander-SAE} & \shortstack{Dense-SAE\\(full $n$)} & \multicolumn{2}{c}{Matched-$n'$ Dense-SAE} \\
        \cmidrule(lr){3-5}\cmidrule(lr){7-8}
        Layer & metric & $d{=}7$ & $d{=}30$ & $d{=}102$ & $n{=}16{,}384$ & $n'{=}240$ & $n'{=}816$ \\
        \midrule
        \multirow{2}{*}{$12$} & rel err & $0.764$ & $0.703$ & $0.681$ & $0.489$ & $0.608$ & $0.556$ \\
                              & CE rec  & $0.828$ & $0.894$ & $0.919$ & $0.983$ & $0.906$ & $0.945$ \\
        \midrule
        \multirow{2}{*}{$24$} & rel err & $0.693$ & $0.848^\star$ & $0.659$ & $0.591$ & $0.541$ & $0.488$ \\
                              & CE rec  & $0.875$ & $0.821^\star$ & $0.884$ & $0.936$ & $0.911$ & $0.946$ \\
        \bottomrule
    \end{tabular}
\end{table}

The monotone $d \mapsto$ fidelity frontier survives the jump to a 3B-parameter modern decoder-only LM with a different architecture. At layer $12$, CE-recovered scales as $0.83 \to 0.89 \to 0.92 \to 0.98$ with $d \in \{7, 30, 102, m\}$, and at $d{=}7$ Expander retains $84\%$ of the full Dense-SAE's recovery while using a decoder $293\times$ smaller in learned values ($114{,}688$ against $33{,}554{,}432$). Layer $24$ shows the same qualitative pattern, with CE-recovered $0.88$ at $d{=}7$ rising to $0.94$ at full Dense-SAE. The matched-parameter dense baselines at $n' \in \{240, 816\}$ score $0.91$ and $0.95$ at layer $12$, and $0.91$ and $0.95$ at layer $24$, which sits moderately above the corresponding Expander rows. At this scale, a fully-flexible dense decoder with the same parameter count is therefore somewhat better in pure CE-recovered terms, and the case for Expander supports rests on three structural advantages. First, Expander masks are compatible with the $(\texttt{values}, \texttt{rows})$ flat layout that makes structured-OMP-style decoding tractable (Appendix~\ref{app:omp}). Second, the mask is regenerable from an $8$-byte seed (Table~\ref{tab:bench_storage}). Third, Expander is Pareto-optimal at the smallest $d$, where the matched-parameter dense baseline is excluded by the NSP gate and the dense full-$n$ baseline is $1170\times$ larger. Two layer-$24$-specific observations follow. First, Dense-SAE at $n'{=}816$ outperforms Dense-SAE at full $n{=}16{,}384$ on rel-err ($0.488$ against $0.591$), suggesting the residual stream at this layer is intrinsically lower-rank than at layer $12$ and that overcomplete dictionaries waste capacity. Second, one of the three seeds at $d{=}30$ diverges during training (rel err $1.13$ on both the original run and a retrain with the same seed), which is included in the reported mean. Whether the matched-dense gap closes at $7$B+ models, and whether seed instability at deep Qwen layers extends to other model families, is left for future work.

\section{Controls: full results}
\label{app:controls_full}

This appendix reports the full per-control results for the support-structure controls summarised in Section~\ref{sec:exp_controls}, and explains why our primary controls fix $(m,n,k)$ rather than the learned-value count.

Parameter-matched dense SAEs may seem like the natural baseline, but matching learned decoder values forces a dense SAE to use only $n'_{\mathrm{dense}} = dn/m$ features, which shifts $\delta = m/n$, changes the dictionary overcompleteness, and sometimes makes the dictionary undercomplete or smaller than the target sparsity level. Sparse-recovery behaviour depends on the problem geometry, with compressed-sensing phase transitions described asymptotically at fixed $\delta = m/n$ and $\rho = k/m$~\citep{donoho2009observed}. Our primary controls therefore preserve $(m,n,k,d)$ and change only the support structure, using clustered sparse masks and pruned-retuned dense supports. The reduced-$n$ Dense-SAE comparison in Appendix~\ref{app:controls_matched} is reported as a sanity check rather than a primary baseline.

\subsection{Clustered-sparse and the dead-feature pathology}
\label{app:controls_clustered}

This subsection isolates whether the low dead-feature rate of the Expander mask comes from column sparsity itself or from the diversity of the tied supports. We add a Clustered-sparse control matched to Expander-SAE in $(m,n,k,d)$, learned-decoder-value count, optimiser, training steps, and seeds, differing only in the mask. The $m$ rows are partitioned into $G = \lfloor m/d \rfloor$ disjoint blocks $B_0, \ldots, B_{G-1}$ of size $d$, and each column $j$ is assigned to one block uniformly at random with $\supp(\mathbf{w}_j) = B_{g_j}$. Columns in the same block therefore share an identical support, which breaks the expansion property since a size-$s$ subset of columns concentrated in one block contains only $d$ distinct neighbours, so $\widehat{\varepsilon}_{\mathrm{greedy}} \to 1 - 1/s$. We measure dead-feature fraction on the same trained grid as Appendix~\ref{app:cross_model_full} ($m{=}512$, $n{=}4096$, $d \in \{7, 50, 200\}$) over the same held-out activations, counting a feature as dead if it never fires under the convention $x_j \neq 0$.

\begin{figure}[h]
    \centering
    \includegraphics[width=\linewidth]{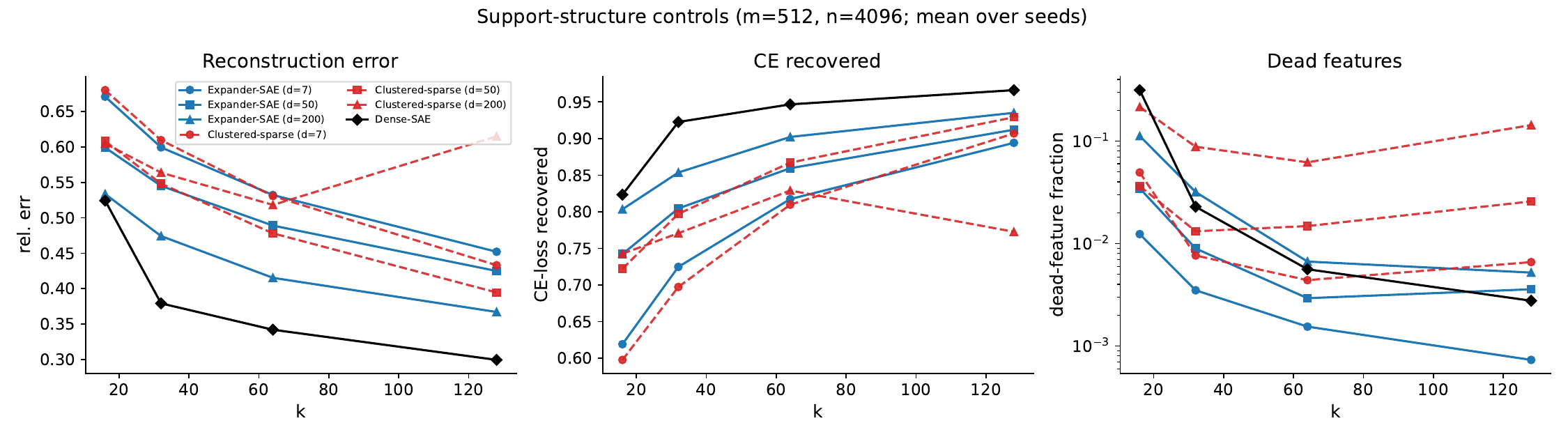}
    \caption{Expander-SAE against Clustered-sparse on three metrics at matched $(d, k)$ with $m{=}512$, $n{=}4096$, trained encoder. Solid lines are Expander-SAE, dashed lines are Clustered-sparse at the matching $d$, and black is the Dense-SAE reference. The left panel shows reconstruction error, the middle panel shows corrected CE-loss recovered, and the right panel shows dead-feature fraction on a log scale.}
    \label{fig:support_structure}
\end{figure}

Expander-SAE acts as a structural regulariser against dead features. Whenever the decoder column for feature $j$ aligns with any sample's residual, the encoder row for $j$ reads activations on the same $d$ neurons and picks it up immediately, so dead-feature rates remain low across $d$. Clustered-sparse is competitive on reconstruction at small and medium $d$, but destroying support diversity drives a feature-usage pathology at large $d$, with the dead-feature rate at $d{=}200$ rising by roughly $100\times$ relative to Expander-SAE at the same column sparsity. Column sparsity alone therefore does not explain the parameter-efficiency frontier, and the diversity of the support also matters.

\subsection{Pruned-retuned dense}
\label{app:controls_pruned}

This subsection tests whether a support extracted from a pre-trained dense decoder carries a more useful geometry than a random $d$-regular expander. The Pruned-retuned dense control starts from a trained dense-tied SAE at the matching $(m, n, k, \text{seed})$, forms the sparse mask by keeping the $d$ rows of largest $|w_{ij}|$ per decoder column, initialises the sparse decoder values from the retained dense weights, column-normalises, and fine-tunes only the retained values for $5000$ steps under the Expander training recipe.

\begin{figure}[h]
\centering
\includegraphics[width=\linewidth]{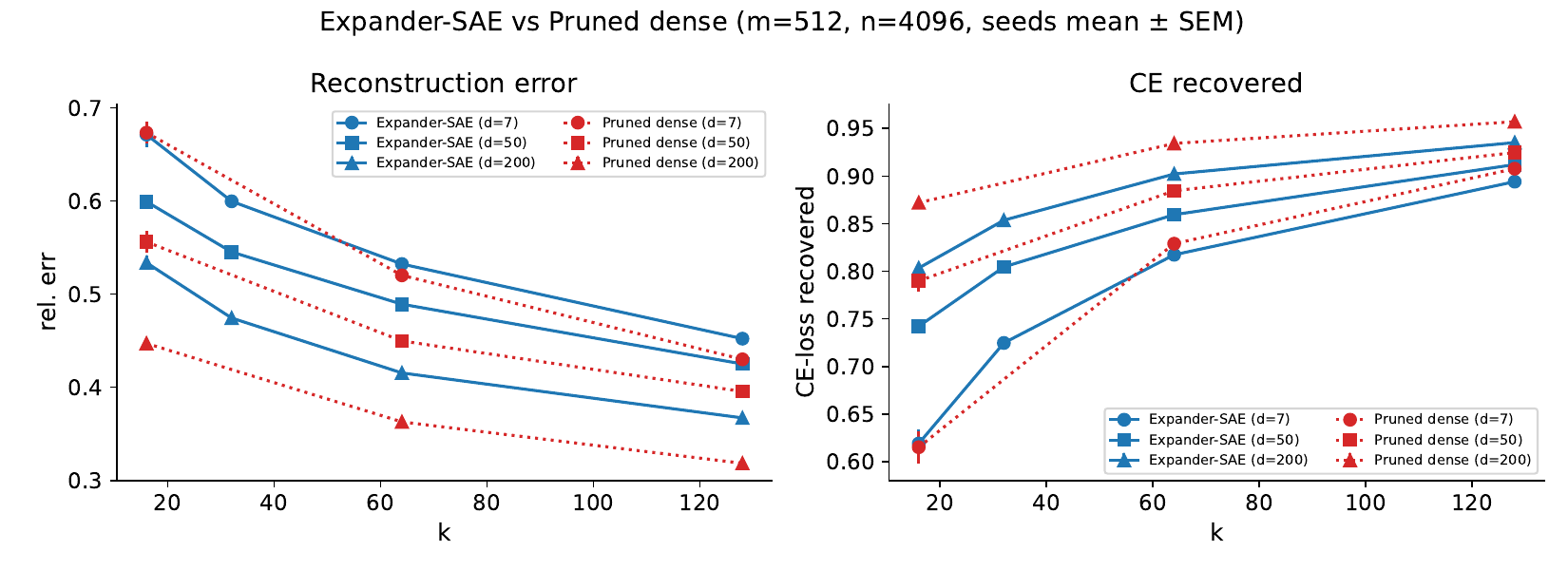}
\caption{Expander-SAE against Pruned dense at matched $(d, k)$, three seeds (mean $\pm$ SEM). Solid lines are Expander-SAE and dotted lines are Pruned dense. Panels show trained-encoder reconstruction error and corrected CE-loss recovered.}
\label{fig:bench_pruned}
\end{figure}

Supports extracted from a trained dense SAE close most of the reconstruction gap to dense at every $d$ tested, but require a full dense pretraining pass on top of the sparse retune. The dense-derived support therefore yields a real geometry advantage, paid for in two training passes.

\subsection{Reduced-$n$ dense at matched parameter budget}
\label{app:controls_matched}

This subsection reports the parameter-matched comparison that the opening discussion flags as not isolating the support effect. Reducing $n$ to match learned parameters shifts $\delta = m/n$, makes the dictionary undercomplete for small $d$, and can drop $n$ below the target sparsity $k$, so the comparison conflates support structure with problem geometry. We run it as a budget-matched sanity check. For each Expander-SAE $(m{=}512, n{=}4096, d)$ at $k{=}64$ we train three reduced-$n$ dense controls at the matching learned-parameter budget. Dense-tied uses $n_{\mathrm{tied}} = dN/M$ to match Expander's $dn$ values against tied dense's $mn$, while Dense-SAE (independent encoder, warm-tied init) and Dense rand-init (independent encoder, random init) both use $n_{\mathrm{indep}} = dN/(2M)$ since they have $2mn$ parameters. All controls use the same optimiser, learning-rate schedule, and held-out evaluation as Appendix~\ref{app:cross_model_full}, and configurations that fall below the null-space property $n > 2k$ are excluded.

\begin{figure}[h]
    \centering
    \includegraphics[width=\linewidth]{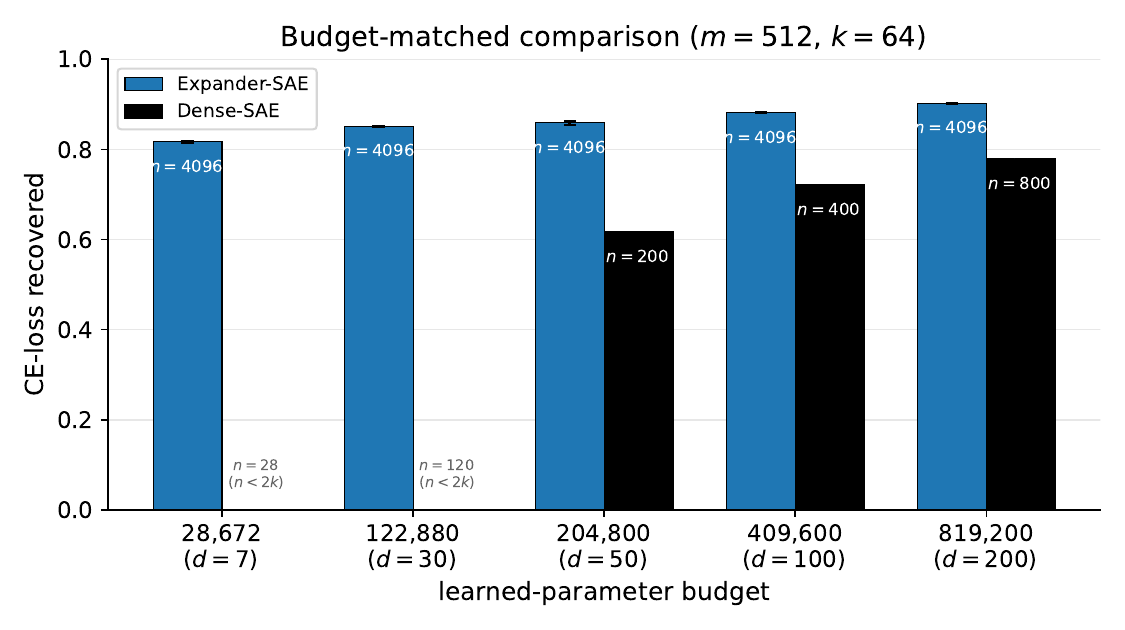}
    \caption{CE-loss recovered at $k{=}64$ versus learned-parameter budget (trained encoder). Each group of bars sits at the budget shared by an Expander-SAE configuration ($d \in \{7, 30, 50, 100, 200\}$ at $n{=}4096$) and three dense baselines reduced to the same budget by lowering $n$, annotated above each bar. Bars show min/max range across seeds.}
    \label{fig:bench_matched}
\end{figure}

At every shared budget, Expander-SAE recovers more CE than the reduced-$n$ dense baselines, with gaps of $0.10$ to $0.21$. The parameter saving is therefore not equivalent to running a smaller dense SAE, since reducing $n$ shrinks the dictionary along recovery-relevant axes ($\delta$, overcompleteness, and sometimes $n < k$) while the Expander keeps the full $n{=}4096$ dictionary and reduces only the per-column degree.

\section{Full storage breakdown and empirical metrics}
\label{app:storage_practical_full}

This appendix collects the complete storage breakdown and per-architecture empirical metrics referenced from the practical summary in Section~\ref{sec:exp_practical}, and records the parameter choices for the Clustered-sparse and Pruned-retuned dense controls.

The block count $G = \lfloor m/d \rfloor$ for the Clustered-sparse mask varies substantially across the three trained degrees. At $m{=}512$, the clustered mask yields $G{=}73$ disjoint blocks at $d{=}7$ with one row unused, $G{=}10$ blocks at $d{=}50$ with $12$ rows unused, and $G{=}2$ blocks at $d{=}200$ with $112$ rows unused. With $n{=}4096$ features and columns assigned to blocks uniformly, each block is shared in expectation by $n/G \approx 56$ columns at $d{=}7$, $\approx 410$ at $d{=}50$, and $\approx 2{,}048$ at $d{=}200$. The Clustered-sparse $d{=}200$ model therefore has only two possible column supports in the whole dictionary, so approximately half of all features share an identical $200$-row support with any randomly chosen reference column. This collapse of support diversity at large $d$ is intentional, since it lets us check whether support diversity helps avoid clustered-mask pathologies at large $d$ or whether the size $d$ of individual supports alone explains the observed parameter-efficiency frontier. We train Clustered-sparse at $d \in \{7, 50, 200\}$, $k \in \{16, 32, 64, 128\}$, three seeds, with the same $5000$-step cosine-schedule optimiser settings as Expander-SAE-d (Appendix~\ref{app:training_details}). Pruned-retuned dense is trained at $d \in \{7, 50, 200\}$, $k \in \{16, 64, 128\}$, three seeds, using the same per-seed dense-tied SAE as its source.

\begin{table}[h]
    \centering
    \caption{Storage breakdown at $m{=}512$, $n{=}4096$ for each architecture and $d$, in KiB unless otherwise noted. \emph{Learned values} counts the unique decoder weights only ($d \cdot n$ float32 entries for the sparse architectures and $m \cdot n$ for Dense-SAE and Expander $d{=}m$). \emph{Decoder + rows} is the on-disk decoder storage including int32 row indices in the $(d \cdot n,)$ flat layout used by the structured kernels, with no \texttt{indptr} array since every column has exactly $d$ nonzeros at known positions; for dense architectures this equals the learned-values column. \emph{Encoder + biases} is the additional inference storage, namely $\mathbf{b}_{\mathrm{dec}}$ and $\mathbf{b}_{\mathrm{enc}}$ for tied-encoder models plus the full $(n, m)$ encoder weights for Dense-SAE. \emph{Total} adds the previous two columns plus $8$ bytes for the mask seed. \emph{Mask via seed} reports the overhead of regenerating the binary mask deterministically from a single \texttt{int64}, with no mask for Dense-SAE. The \emph{Ratio} column is the learned-values ratio against Expander-SAE at $d{=}m{=}512$, equal to $m/d$ by construction.}
    \label{tab:bench_storage}
    \small
    \setlength{\tabcolsep}{5pt}
    \begin{tabular}{llrrrrrrr}
        \toprule
        Architecture & $d$ & \shortstack{(1) Learned\\values (KiB)} & Ratio & \shortstack{(2) Decoder\\$+$ rows (KiB)} & \shortstack{(3) Encoder\\$+$ biases (KiB)} & \shortstack{(4) Total\\(KiB)} & \shortstack{(5) Mask\\via seed} \\
        \midrule
        Dense-SAE                & $512$       & $8{,}192.0$ & $1\times$ & $8{,}192.0$ & $8{,}210.0$ & $16{,}402.0$ & --   \\
        \midrule
        \multirow{6}{*}{Expander-SAE}
                                 & $7$         & $\mathbf{112.0}$  & $\mathbf{73\times}$ & $\mathbf{224.0}$  & $18.0$ & $\mathbf{242.0}$  & $8$ B \\
                                 & $30$        & $480.0$    & $17\times$  & $960.0$    & $18.0$ & $978.0$    & $8$ B \\
                                 & $50$        & $800.0$    & $10\times$  & $1{,}600.0$ & $18.0$ & $1{,}618.0$ & $8$ B \\
                                 & $100$       & $1{,}600.0$ & $5\times$   & $3{,}200.0$ & $18.0$ & $3{,}218.0$ & $8$ B \\
                                 & $200$       & $3{,}200.0$ & $2.6\times$ & $6{,}400.0$ & $18.0$ & $6{,}418.0$ & $8$ B \\
                                 & $512\,(=m)$ & $8{,}192.0$ & $1\times$   & $8{,}192.0$ & $18.0$ & $8{,}210.0$ & $8$ B \\
        \midrule
        \multirow{3}{*}{Clustered-sparse}
                                 & $7$            & $112.0$    & $73\times$  & $224.0$    & $18.0$ & $242.0$    & $8$ B \\
                                 & $50$           & $800.0$    & $10\times$  & $1{,}600.0$ & $18.0$ & $1{,}618.0$ & $8$ B \\
                                 & $200$          & $3{,}200.0$ & $2.6\times$ & $6{,}400.0$ & $18.0$ & $6{,}418.0$ & $8$ B \\
        \midrule
        \multirow{3}{*}{Pruned dense}
                                 & $7$         & $112.0$    & $73\times$  & $224.0$    & $18.0$ & $242.0$    & $8$ B \\
                                 & $50$        & $800.0$    & $10\times$  & $1{,}600.0$ & $18.0$ & $1{,}618.0$ & $8$ B \\
                                 & $200$       & $3{,}200.0$ & $2.6\times$ & $6{,}400.0$ & $18.0$ & $6{,}418.0$ & $8$ B \\
        \bottomrule
    \end{tabular}
\end{table}

\begin{table}[h]
    \centering
    \caption{Empirical metrics at $m{=}512$, $n{=}4096$, $k{=}64$ on Pythia-70M layer 3 (CE clean $3.44$, CE zero-ablation $12.73$). All metrics use the trained encoder, with entries reporting mean $\pm$ SEM over three seeds except where marked $\dagger$ (single seed) or $\ddagger$ (two seeds) due to the original sweep grid. \emph{Training passes} counts the full SAE training runs needed, with Pruned dense requiring a dense pretraining pass before the sparse retune. \emph{Novel ($J{<}0.1$)} is the fraction of features whose best-match activation-Jaccard against Dense-SAE is below $0.1$ from the single-seed feature-analysis pipeline. The corresponding storage breakdown is in Table~\ref{tab:bench_storage}.}
    \label{tab:bench_practical}
    \small
    \setlength{\tabcolsep}{5pt}
    \begin{tabular}{llrrrrr}
        \toprule
        Architecture & $d$ & \shortstack{Training\\passes} & \shortstack{rel\\err} & \shortstack{CE\\rec.} & \shortstack{Dead\\frac} & \shortstack{Novel$^\dagger$\\($J{<}0.1$)} \\
        \midrule
        Dense-SAE                & $512$       & $1$          & $0.342{\pm}0.000$ & $0.947{\pm}0.001$ & $0.6\%{\pm}0.0$  & --                \\
        \midrule
        \multirow{6}{*}{Expander-SAE}
                                 & $7$         & $1$          & $0.541{\pm}0.006$ & $0.817{\pm}0.001$ & $0.1\%{\pm}0.0$  & $\mathbf{81.1\%}$ \\
                                 & $30$        & $1$          & $0.508{\pm}0.001$ & $0.851{\pm}0.001$ & $0.2\%{\pm}0.0$  & $72.6\%$          \\
                                 & $50$        & $1$          & $0.489{\pm}0.003$ & $0.859{\pm}0.002$ & $0.3\%{\pm}0.0$  & $69.4\%$          \\
                                 & $100$       & $1$          & $0.467{\pm}0.003$ & $0.882{\pm}0.000$ & $0.3\%{\pm}0.1$  & $66.8\%$          \\
                                 & $200$       & $1$          & $0.415{\pm}0.001$ & $0.902{\pm}0.001$ & $0.7\%{\pm}0.0$  & $65.5\%$          \\
                                 & $512\,(=m)$ & $1$          & $0.388{\pm}0.008$ & $0.937{\pm}0.000$ & $0.3\%{\pm}0.1$  & $5.8\%$           \\
        \midrule
        \multirow{3}{*}{Clustered-sparse}
                                 & $7$         & $1$          & $0.531{\pm}0.001$ & $0.810{\pm}0.001$ & $0.4\%{\pm}0.0$  & $82.1\%$          \\
                                 & $50$        & $1$          & $0.478{\pm}0.001$ & $0.867{\pm}0.000$ & $1.5\%{\pm}0.2$  & $78.4\%$          \\
                                 & $200^\ddagger$ & $1$       & $0.518{\pm}0.012$ & $0.829{\pm}0.016$ & $\mathbf{6.2\%{\pm}0.3}$ & $77.3\%$ \\
        \midrule
        \multirow{3}{*}{Pruned dense}
                                 & $7$         & $\mathbf{2}$ & $0.520{\pm}0.001$ & $0.829{\pm}0.001$ & $0.3\%{\pm}0.1$  & $82.3\%$          \\
                                 & $50$        & $\mathbf{2}$ & $0.450{\pm}0.004$ & $0.885{\pm}0.005$ & $0.4\%{\pm}0.1$  & $62.2\%$          \\
                                 & $200$       & $\mathbf{2}$ & $0.363{\pm}0.002$ & $0.934{\pm}0.001$ & $0.4\%{\pm}0.0$  & $15.1\%$          \\
        \bottomrule
    \end{tabular}
\end{table}

Across the trained grid, Expander-SAE-d delivers substantial decoder-storage reductions and high CE-loss recovery in a single training pass. The controls clarify which gains come from column sparsity, support diversity, or a dense-derived support, and which require additional training.

\section{CE-loss recovered protocol}
\label{app:ce_protocol}

This appendix specifies the CE-loss recovered evaluation used as the downstream functional check throughout the paper. For each evaluation sequence we run the clean language model and record the next-token cross-entropy over non-padding token positions, then run two hooked evaluations at the same layer and hook site used for SAE training. In the zero-ablation run, the layer-$\ell$ residual-stream activation at the hook site is replaced by the all-zero vector. In the reconstruction run, the same activation is replaced by the SAE reconstruction $\hat{\mathbf{h}} = \mathbf{W}_{\mathrm{dec}}\mathbf{x} + \mathbf{b}_{\mathrm{dec}}$. The reported metric is
\[
    \mathrm{CE\ recovered}
    \;=\;
    \frac{
        \mathrm{CE}_{\mathrm{zero}} - \mathrm{CE}_{\mathrm{recon}}
    }{
        \mathrm{CE}_{\mathrm{zero}} - \mathrm{CE}_{\mathrm{clean}}
    },
\]
where $1$ matches the clean model's CE and $0$ matches zero-ablation. The CE evaluation uses $1000$ held-out Pile sequences of length $128$, disjoint from the sequences used to build the SAE training activation cache. The released experiment configuration specifies the literal hook name, whether padding tokens are excluded from CE, and whether the hook is applied to every token position or only to non-padding positions.

\section{Decoder-cosine novelty}
\label{app:decoder_cosine_novelty}

This appendix reports decoder-cosine novelty as a complement to the activation-Jaccard novelty in Appendix~\ref{app:novelty}. A $d$-sparse decoder column has bounded cosine similarity to a dense column whose mass is spread evenly across coordinates, irrespective of whether the two represent the same feature, so the metric is biased downward at small $d$ for purely geometric reasons. At $d{=}m$ the bias vanishes and the cosine and Jaccard metrics agree.

\begin{figure}[h]
    \centering
    \includegraphics[width=0.6\linewidth]{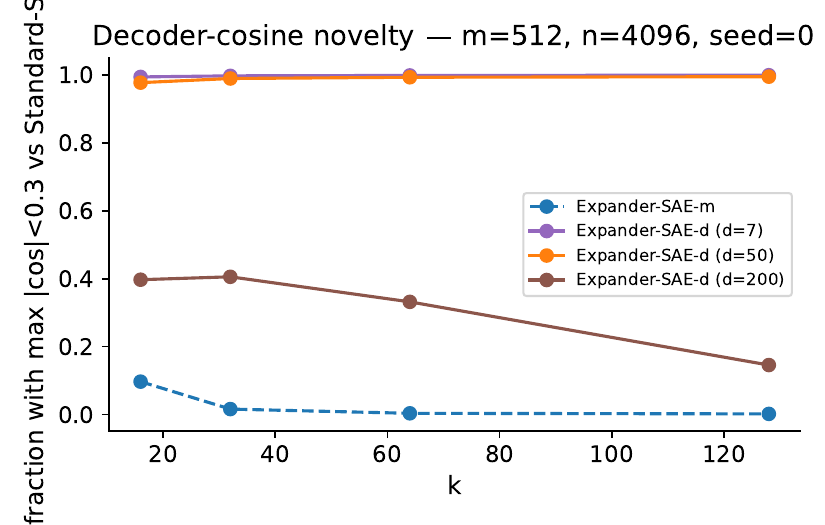}
    \caption{Fraction of Expander decoder columns whose maximum absolute cosine against any Dense-SAE decoder column is below $0.1$.}
    \label{fig:bench_decoder_cosine}
\end{figure}

\section{Synthetic sparse-recovery experiment}
\label{app:synthetic_omp}

This appendix reports a classical compressed-sensing sanity study on $d$-regular Expander sensing matrices, decoupled from the language-model activation setting, that isolates the recovery behaviour of OMP on the Expander support geometry alone. We fix the decoder values analytically so the column-flatness factor $\beta(\mathbf{W})$ is minimised by construction and there is no encoder to train, and we draw measured signals from the exact model the theorem assumes. Recovery failure at a given $(d, k_{\mathrm{synth}})$ is therefore attributable to the mask rather than to suboptimal learned weights. For each $(d, k_{\mathrm{synth}}, \text{mask seed}, \sigma)$ we sample a $d$-regular bipartite mask $\mathbf{M}$ and construct a balanced signed decoder with $\mathbf{W}_{ij} = \pm 1/\sqrt{d}$ on the mask support under independent random signs, and zero elsewhere. On each trial we sample a $k_{\mathrm{synth}}$-sparse signal $\mathbf{x}_S \sim \mathcal{N}(0, I)$ on a uniformly drawn support $S$ of size $k_{\mathrm{synth}}$, form the measurement $\mathbf{h} = \mathbf{W}\mathbf{x}$, optionally add noise $\sigma \cdot (\|\mathbf{W}\mathbf{x}\|_2 / \sqrt{m}) \cdot \mathbf{z}$ with $\mathbf{z} \sim \mathcal{N}(0, I)$, and run OMP for exactly $k_{\mathrm{synth}}$ iterations to record whether the recovered support equals $S$. We fix $m{=}512$, $n{=}4096$, $100$ trials per configuration, $3$ mask seeds, and $\sigma \in \{0, 0.01, 0.05\}$.

\begin{figure}[h]
\centering
\includegraphics[width=\linewidth]{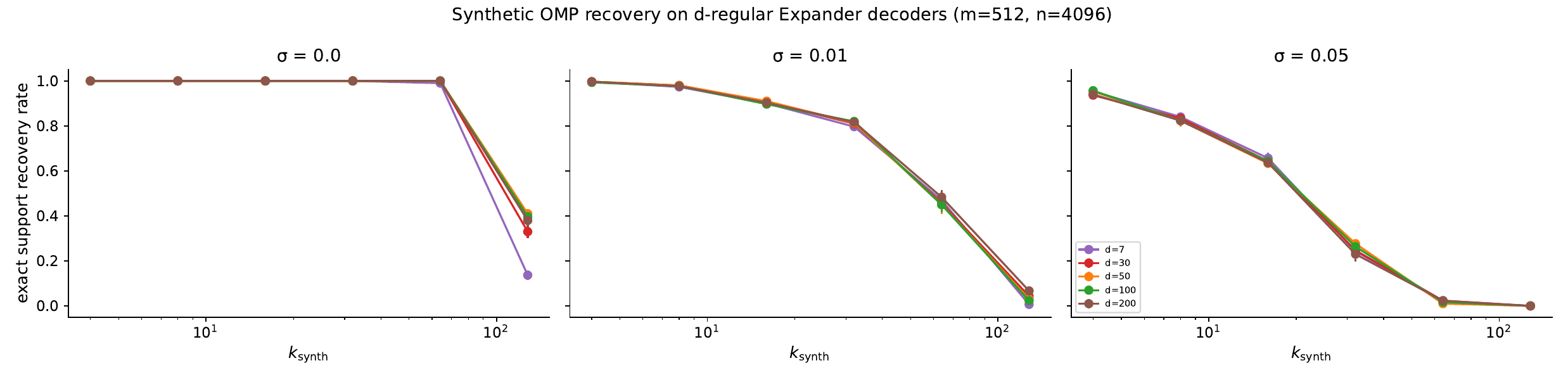}
\caption{Exact support-recovery rate as a function of $k_{\mathrm{synth}}$, with one curve per Expander degree $d \in \{7, 30, 50, 100, 200\}$ and one panel per noise level $\sigma$. Points are means over three mask seeds and error bars show the standard error.}
\label{fig:synthetic_omp_recovery}
\end{figure}

\begin{figure}[h]
\centering
\includegraphics[width=\linewidth]{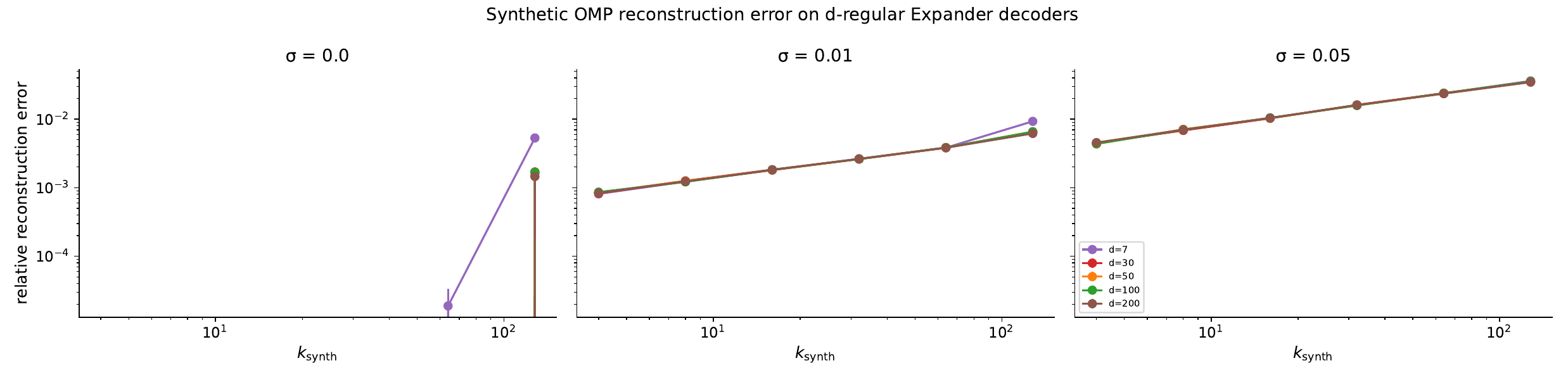}
\caption{Relative reconstruction error $\|\mathbf{W}\hat{\mathbf{x}} - \mathbf{W}\mathbf{x}\|_2 / \|\mathbf{W}\mathbf{x}\|_2$ for the same synthetic experiment as Figure~\ref{fig:synthetic_omp_recovery}, on log-log axes. Points are means over three mask seeds and error bars show the standard error.}
\label{fig:synthetic_omp_recon}
\end{figure}

\section{Feature Dashboards}
\label{app:features}

This appendix presents a qualitative look at individual features learned by the Expander SAE at $d{=}30$.
We sample features from two regimes: \emph{novel} features, defined as having best-match Jaccard overlap below $0.1$ against the dense SAE baselines considered in Appendix~\ref{app:novelty}, and \emph{shared} features with best-match Jaccard above $0.5$ against some dense feature.
For each feature we report the Feature ID, a heuristic category (\emph{neuron-pattern}, \emph{positional}, \emph{other}, or \emph{incoherent}) derived from the entropy of its firing positions, preceding tokens, target tokens, and the co-activation pattern of the $d$ underlying neurons; the best-match Jaccard overlap against any dense-baseline feature; the split-half firing-rate ratio as a reliability check; the top target tokens by firing frequency; and a small number of sample contexts, with the target token highlighted in brackets.
An interpretive note follows each feature where we can offer one.

\subsection*{Novel features (best-match Jaccard $<0.1$)}

\paragraph{Feature 2113 (neuron-pattern).}
Jaccard $0.085$; rate ratio $0.902$.
Top targets: \texttt{\textbackslash n}, \texttt{>}, ``\,don'', \texttt{();}, ``\,a''.
Sample contexts (target in brackets):
\begin{itemize}
    \item \texttt{...\ "scale" : "2x", [",]}
    \item \texttt{...; return 2; end;; f(); Where[();]}
    \item \texttt{...screen-width" : ">145" ["]}
    \item \texttt{...documentation.[ers]}
\end{itemize}
No single dominant target; the underlying $30$ neurons co-activate at well above the random-co-firing baseline, which is what places it in the \emph{neuron-pattern} category.

\paragraph{Feature 1834 (neuron-pattern).}
Jaccard $0.082$; rate ratio $0.977$.
Top targets: ``\,don'' ($24$), ``\,doesn'' ($22$), \texttt{,} ($20$), \texttt{-} ($14$), ``\,number'' ($11$).
Top preceding tokens: ``\,the'' ($35$), ``\,it'' ($12$), \texttt{.} ($10$).
Sample contexts:
\begin{itemize}
    \item \texttt{...\ it simply [doesn]'t give me}
    \item \texttt{...A: SEDE data [ isn]'t live}
    \item \texttt{...with a single font and it [ doesn]'t turn to}
    \item \texttt{...this [ isn]'t a bigger}
\end{itemize}
Plausibly the ``second piece of an English contraction whose first piece is \emph{n}+\emph{not}'' (e.g.\ the \texttt{doesn} in \emph{doesn't}, the \texttt{isn} in \emph{isn't}).
This feature is mostly missed by the dense SAE.

\paragraph{Feature 1587 (other).}
Jaccard $0.091$; rate ratio $0.847$.
Top targets: ``\,of'' ($62$), ``\,the'' ($29$), ``\,for'' ($10$).
Top preceding tokens: ``\,of'' ($14$), ``\,some'' ($9$), ``\,one'' ($8$).
Plausibly an ``of-phrase head noun'' detector, firing on the content word immediately following an \emph{of}-preposition.
Its absence from the dense SAE baselines considered here is consistent with the interpretation that the Expander column geometry supports a parameterisation that the dense SAE does not preferentially select.

\paragraph{Feature 3248 (incoherent).}
Jaccard $0.091$; rate ratio $0.995$; target entropy $6.62$ bits.
Fires on a highly diverse set of targets (\texttt{d}, \texttt{\textbackslash n}, \texttt{int}, \texttt{lib}, ``\,F'') with no consistent left context.
The high rate ratio indicates the firing pattern is reproducible across data splits, but none of our heuristics identifies a dominant concept.

\paragraph{Feature 3045 (incoherent).}
Jaccard $0.089$; rate ratio $0.970$; target entropy $5.98$ bits.
Top targets are punctuation and digits (\texttt{\_}, \texttt{.}, \texttt{2}).
Plausibly a ``code-like context'' detector, but the target distribution is too diffuse to isolate a single concept under our metrics.

\subsection*{Shared features (best-match Jaccard $>0.5$)}

\paragraph{Feature 2967 (neuron-pattern).}
Jaccard $0.962$; target entropy $1.45$ bits.
Top targets: \texttt{.} ($110$), \texttt{\textbackslash n} ($78$).
A clean ``end-of-sentence'' detector.
The dense SAE baselines considered here contain a near-identical feature, as expected.

\paragraph{Feature 1500 (neuron-pattern).}
Jaccard $>0.5$.
Matches a dense-SAE feature cleanly and is classified as neuron-pattern under our heuristics.

\paragraph{Feature 2909 (positional).}
Jaccard $0.929$; target entropy $6.32$ bits, but position entropy is below threshold.
Top targets: ``\,In'' ($11$), ``\,M'' ($10$), \texttt{\#} ($9$), ``\,package'' ($9$), ``\,F'' ($8$).
Sample contexts:
\begin{itemize}
    \item \texttt{...Every industry has its [Every]}
    \item \texttt{...Serum alpha-[Ser]}
    \item \texttt{...Debbie Gregory D[Deb]}
    \item \texttt{...Frederick L[Fred]}
    \item \texttt{...Buy Nureflex [Buy]}
\end{itemize}
A ``beginning-of-token'' detector firing at sentence-initial positions.
The high target-token entropy but low position entropy is the signature that places this feature in the \emph{positional} rather than the \emph{neuron-pattern} category.

\section{Theory for Expander SAE decoders}
\label{app:expander_sae_theory}

This appendix proves Theorem~\ref{thm:expander_sae_recovery} via a weighted version of the standard expander recovery argument. The proof has two ingredients. The binary mask $\mathbf{M}$ controls how many collisions occur between feature supports, and the flatness factor $\beta(\mathbf{W}_{\mathrm{dec}})$ controls how much learned decoder mass can sit on those colliding edges. We first establish a combinatorial bound on collisions using the expansion property, then show that column normalisation pins down the $\ell_1$-mass of each decoder column, and finally combine the two to prove the identifiability theorem and its corollaries.

Throughout this appendix we write $\mathbf{W} = \mathbf{W}_{\mathrm{dec}}$ for brevity. We assume the columns of $\mathbf{W}$ are unit-normalised and supported on a left $d$-regular mask $\mathbf{M}$, so that
\[
    \|\mathbf{w}_j\|_2 = 1,
    \qquad
    \mathbf{W}_{ij} = 0 \text{ whenever } \mathbf{M}_{ij} = 0.
\]
For $J \subseteq [n]$, the neighbourhood of $J$ in the bipartite graph defined by $\mathbf{M}$ is
\[
    \Gamma(J)
    \;:=\;
    \{\,i \in [m] : \mathbf{M}_{ij} = 1 \text{ for some } j \in J\,\},
\]
and the flatness factor is
\[
    \beta(\mathbf{W})
    \;:=\;
    \sqrt{d}\,
    \max_{i,j:\,\mathbf{M}_{ij} = 1}
    |\mathbf{W}_{ij}|.
\]

\subsection{Expansion gives few collisions}

This subsection translates the expansion property of $\mathbf{M}$ into a bound on the cumulative number of collision edges that appear when we add columns to a set one at a time. The argument is purely combinatorial and uses only the expansion of $\mathbf{M}$, not the values of $\mathbf{W}$.

Let $S \subseteq [n]$ with $|S| \le s$, and order its elements as
\[
    j_1, j_2, \dots, j_{|S|}.
\]
For $r = 1, \dots, |S|$, define the prefix
\[
    S_r \;:=\; \{j_1, \dots, j_r\}
\]
and the set of collision edges created by $j_r$ relative to previous columns,
\[
    B_r
    \;:=\;
    \Gamma(j_r) \cap \Gamma(S_{r-1}),
    \qquad
    b_r \;:=\; |B_r|,
\]
so $b_r$ counts the neighbours of $j_r$ that were already used by earlier columns in the prefix.

\begin{lemma}[Prefix collision bound]
\label{lem:prefix_collision_bound}
If $\mathbf{M}$ is a left $d$-regular $(s, \varepsilon, d)$-expander, then for every ordering of every $S \subseteq [n]$ with $|S| \le s$,
\[
    \sum_{q=1}^r b_q
    \;\le\;
    \varepsilon d r
    \qquad
    \text{for every } r \le |S|.
\]
\end{lemma}

\begin{proof}
When column $j_q$ is added to the prefix, it contributes $d - b_q$ new neighbours. Hence
\[
    |\Gamma(S_r)|
    \;=\;
    \sum_{q=1}^r (d - b_q)
    \;=\;
    dr - \sum_{q=1}^r b_q.
\]
By expansion,
\[
    |\Gamma(S_r)|
    \;\ge\;
    (1 - \varepsilon) dr.
\]
Therefore
\[
    dr - \sum_{q=1}^r b_q
    \;\ge\;
    (1 - \varepsilon) dr,
\]
which gives
\[
    \sum_{q=1}^r b_q
    \;\le\;
    \varepsilon d r.
\]
\end{proof}

The prefix bound controls collision counts uniformly over $r$, but the proof of the main theorem needs a weighted version in which each collision is reweighted by an entry of the latent vector. The next lemma provides that weighted bound by combining Lemma~\ref{lem:prefix_collision_bound} with summation by parts.

\begin{lemma}[Weighted collision bound]
\label{lem:weighted_collision_bound}
Let $a_1 \ge a_2 \ge \cdots \ge a_s \ge 0$. Under the assumptions of Lemma~\ref{lem:prefix_collision_bound},
\[
    \sum_{r=1}^s b_r a_r
    \;\le\;
    \varepsilon d
    \sum_{r=1}^s a_r.
\]
\end{lemma}

\begin{proof}
Let
\[
    C_r \;:=\; \sum_{q=1}^r b_q.
\]
By Lemma~\ref{lem:prefix_collision_bound},
\[
    C_r \;\le\; \varepsilon d r.
\]
Set $a_{s+1} = 0$. By summation by parts,
\[
    \sum_{r=1}^s b_r a_r
    \;=\;
    \sum_{r=1}^s C_r (a_r - a_{r+1}).
\]
Therefore
\[
    \sum_{r=1}^s b_r a_r
    \;\le\;
    \varepsilon d
    \sum_{r=1}^s r (a_r - a_{r+1})
    \;=\;
    \varepsilon d
    \sum_{r=1}^s a_r.
\]
\end{proof}

\subsection{Flatness bounds learned weighted cancellations}

The previous subsection treats $\mathbf{M}$ as a combinatorial object and ignores the values of $\mathbf{W}$. This subsection extracts the two consequences of the flatness factor that the main proof needs, both following from $\beta(\mathbf{W})$ together with column normalisation. The first bounds the magnitude of any single colliding entry, and the second lower-bounds the total $\ell_1$-mass of every decoder column.

By definition of $\beta(\mathbf{W})$,
\[
    |\mathbf{W}_{ij}|
    \;\le\;
    \frac{\beta(\mathbf{W})}{\sqrt{d}}
    \qquad
    \text{whenever } \mathbf{M}_{ij} = 1,
\]
so each colliding edge contributes at most
\[
    \frac{\beta(\mathbf{W})}{\sqrt{d}}\,|\mathbf{u}_j|
\]
in absolute value, for any latent vector $\mathbf{u}$.

For the second consequence, column normalisation gives a lower bound on the total $\ell_1$-mass of every decoder column. Since $\|\mathbf{w}_j\|_2 = 1$,
\[
    1
    \;=\;
    \sum_i |\mathbf{W}_{ij}|^2
    \;\le\;
    \|\mathbf{w}_j\|_\infty \|\mathbf{w}_j\|_1
    \;\le\;
    \frac{\beta(\mathbf{W})}{\sqrt{d}}\|\mathbf{w}_j\|_1.
\]
Therefore
\begin{equation}
\label{eq:l1_column_lower}
    \|\mathbf{w}_j\|_1
    \;\ge\;
    \frac{\sqrt{d}}{\beta(\mathbf{W})}.
\end{equation}

Equation~\eqref{eq:l1_column_lower} is the point at which the learned non-binary nature of the SAE decoder enters the proof. In the balanced binary case $\beta(\mathbf{W}) = 1$, recovering the lossless-expander setting of \citet{jafarpour2009efficient}. For a learned decoder, $\beta(\mathbf{W})$ measures how much the weighted argument deviates from that ideal balanced case.

\subsection{Proof of Theorem~\ref{thm:expander_sae_recovery}}

The proof has two stages. The first stage establishes a uniform lower bound on $\|\mathbf{W}\mathbf{u}\|_1$ for every $2k$-sparse vector $\mathbf{u}$, by decomposing the row sums of $\mathbf{W}\mathbf{u}$ into a primary contribution $A$ and a colliding-edge correction $B$, and then applying Lemma~\ref{lem:weighted_collision_bound} together with the column $\ell_1$-bound \eqref{eq:l1_column_lower}. The second stage uses this lower bound to deduce that no two distinct $k$-sparse codes can produce the same centred activation, which gives the identifiability claim.

\begin{proof}
Let
\[
    \beta \;:=\; \beta(\mathbf{W}).
\]
We prove that every $\mathbf{u}$ with $\|\mathbf{u}\|_0 \le 2k$ satisfies
\[
    \|\mathbf{W}\mathbf{u}\|_1
    \;\ge\;
    \sqrt{d}
    \left(
        \frac{1}{\beta} - 2\beta\varepsilon
    \right)
    \|\mathbf{u}\|_1.
\]
The claim is trivial when $\mathbf{u} = 0$, so assume $\mathbf{u} \neq 0$.

Let
\[
    S \;:=\; \operatorname{supp}(\mathbf{u}),
    \qquad
    s \;:=\; |S| \le 2k.
\]
Order the elements of $S$ as
\[
    j_1, \dots, j_s
\]
so that
\[
    |\mathbf{u}_{j_1}|
    \;\ge\;
    |\mathbf{u}_{j_2}|
    \;\ge\;
    \cdots
    \;\ge\;
    |\mathbf{u}_{j_s}|,
\]
and set
\[
    a_r \;:=\; |\mathbf{u}_{j_r}|.
\]
For each row $i \in \Gamma(S)$, let $j_{r(i)}$ be the first column in this ordering that touches row $i$. We call the edge $(i, j_{r(i)})$ primary, and call all later edges touching the same row colliding.

For each row $i$, write
\[
    t_{ij} \;:=\; \mathbf{W}_{ij}\mathbf{u}_j.
\]
By the reverse triangle inequality,
\[
    \left|
        \sum_{j \in S} t_{ij}
    \right|
    \;\ge\;
    |t_{i, j_{r(i)}}|
    \;-\;
    \sum_{\substack{j \in S:\,\mathbf{M}_{ij} = 1 \\ j \neq j_{r(i)}}}
    |t_{ij}|.
\]
Equivalently,
\[
    \left|
        \sum_{j \in S} t_{ij}
    \right|
    \;\ge\;
    \sum_{j \in S:\,\mathbf{M}_{ij} = 1} |t_{ij}|
    \;-\;
    2 \sum_{\substack{j \in S:\,\mathbf{M}_{ij} = 1 \\ (i,j) \text{ colliding}}}
    |t_{ij}|.
\]
Summing this inequality over rows gives
\begin{equation}
\label{eq:collision_decomposition}
    \|\mathbf{W}\mathbf{u}\|_1
    \;\ge\;
    A - 2B,
\end{equation}
where
\[
    A
    \;:=\;
    \sum_{j \in S}
    |\mathbf{u}_j|\,\|\mathbf{w}_j\|_1
\]
is the total absolute decoder mass and
\[
    B
    \;:=\;
    \sum_{\text{colliding edges } (i, j)}
    |\mathbf{W}_{ij}\mathbf{u}_j|
\]
is the total absolute mass on colliding edges.

We first lower-bound $A$. By \eqref{eq:l1_column_lower},
\[
    A
    \;=\;
    \sum_{j \in S} |\mathbf{u}_j|\,\|\mathbf{w}_j\|_1
    \;\ge\;
    \frac{\sqrt{d}}{\beta}
    \sum_{j \in S} |\mathbf{u}_j|
    \;=\;
    \frac{\sqrt{d}}{\beta}
    \|\mathbf{u}\|_1.
\]

We next upper-bound $B$. In the ordering above, the colliding edges of column $j_r$ are contained in
\[
    \Gamma(j_r) \cap \Gamma(\{j_1, \dots, j_{r-1}\}),
\]
and the number of such edges is $b_r$. Since
\[
    |\mathbf{W}_{ij}|
    \;\le\;
    \frac{\beta}{\sqrt{d}},
\]
we have
\[
    B
    \;\le\;
    \sum_{r=1}^s
    b_r a_r
    \frac{\beta}{\sqrt{d}}.
\]
By Lemma~\ref{lem:weighted_collision_bound},
\[
    \sum_{r=1}^s b_r a_r
    \;\le\;
    \varepsilon d
    \sum_{r=1}^s a_r
    \;=\;
    \varepsilon d\,\|\mathbf{u}\|_1,
\]
so
\[
    B
    \;\le\;
    \beta\varepsilon\sqrt{d}\,\|\mathbf{u}\|_1.
\]

Substituting the bounds on $A$ and $B$ into \eqref{eq:collision_decomposition} gives
\[
    \|\mathbf{W}\mathbf{u}\|_1
    \;\ge\;
    \frac{\sqrt{d}}{\beta}\|\mathbf{u}\|_1
    \;-\;
    2\beta\varepsilon\sqrt{d}\,\|\mathbf{u}\|_1,
\]
hence
\[
    \|\mathbf{W}\mathbf{u}\|_1
    \;\ge\;
    \sqrt{d}
    \left(
        \frac{1}{\beta} - 2\beta\varepsilon
    \right)
    \|\mathbf{u}\|_1.
\]
If $2\beta^2\varepsilon < 1$, the coefficient is strictly positive.

To deduce identifiability, let $\mathbf{x}_\star$ be $k$-sparse and suppose $\mathbf{z}$ is another $k$-sparse vector with
\[
    \mathbf{W}\mathbf{z} \;=\; \mathbf{W}\mathbf{x}_\star.
\]
Set
\[
    \mathbf{u} \;:=\; \mathbf{z} - \mathbf{x}_\star.
\]
Then $\|\mathbf{u}\|_0 \le \|\mathbf{z}\|_0 + \|\mathbf{x}_\star\|_0 \le 2k$ and $\mathbf{W}\mathbf{u} = 0$. The lower bound gives
\[
    0
    \;=\;
    \|\mathbf{W}\mathbf{u}\|_1
    \;\ge\;
    \sqrt{d}
    \left(
        \frac{1}{\beta} - 2\beta\varepsilon
    \right)
    \|\mathbf{u}\|_1.
\]
Since the coefficient is positive, $\|\mathbf{u}\|_1 = 0$, and therefore $\mathbf{u} = 0$ and $\mathbf{z} = \mathbf{x}_\star$. So $\mathbf{x}_\star$ is the unique $k$-sparse latent vector producing the noiseless centred activation $\mathbf{W}\mathbf{x}_\star$. Since $\mathbf{x}_\star$ is feasible for the ideal sparse decoding problem and achieves zero reconstruction error, every minimiser of \eqref{eq:decoding} must equal $\mathbf{x}_\star$.
\end{proof}

\subsection{Corollaries}
\label{app:theory_corollaries}

Theorem~\ref{thm:expander_sae_recovery} has three immediate consequences. The first translates the uniform $2k$-sparse lower bound into a spark statement, which gives the unique-sparsest-explanation property of the ideal $\ell_0$ decoder. The second derives a cumulative-coherence bound that activates the classical OMP exact-recovery condition. The third extends the identifiability claim to a stable-recovery statement under bounded $\ell_1$-residual mismatch.

\begin{corollary}[Spark and ideal sparse-decoder recovery]
\label{cor:spark_identifiability}
Under the assumptions of Theorem~\ref{thm:expander_sae_recovery},
\[
    \operatorname{spark}(\mathbf{W}_{\mathrm{dec}}) \;>\; 2k.
\]
Consequently, every $k$-sparse code $\mathbf{x}_\star$ is the unique sparsest explanation of the noiseless centred activation $\mathbf{W}_{\mathrm{dec}}\mathbf{x}_\star$, and the ideal $\ell_0$ decoder in Eq.~\eqref{eq:decoding} recovers $\mathbf{x}_\star$ exactly.
\end{corollary}

\begin{proof}
Theorem~\ref{thm:expander_sae_recovery} gives $\|\mathbf{W}_{\mathrm{dec}}\mathbf{u}\|_1 > 0$ for every nonzero $2k$-sparse vector $\mathbf{u}$, so no set of at most $2k$ decoder columns is linearly dependent and $\operatorname{spark}(\mathbf{W}_{\mathrm{dec}}) > 2k$. The sparse-identifiability implication follows. If two $k$-sparse vectors $\mathbf{x}$ and $\mathbf{z}$ satisfy $\mathbf{W}_{\mathrm{dec}}\mathbf{x} = \mathbf{W}_{\mathrm{dec}}\mathbf{z}$, then $\mathbf{u} = \mathbf{x} - \mathbf{z}$ is $2k$-sparse and lies in the nullspace, so $\mathbf{u} = 0$.
\end{proof}

The second corollary moves from existential identifiability to a constructive recovery guarantee. The classical analysis of \citet{tropp2004greed} certifies OMP exact recovery whenever the cumulative coherence satisfies $\mu_1(k; \mathbf{W}) + \mu_1(k-1; \mathbf{W}) < 1$, and we show that the prefix collision bound of Lemma~\ref{lem:prefix_collision_bound} controls $\mu_1(s; \mathbf{W})$ in our setting.

\begin{corollary}[A sufficient cumulative-coherence condition for OMP]
\label{cor:omp_coherence}
Let $\mathbf{W} = \mathbf{W}_{\mathrm{dec}}$ have unit-normalised columns supported on a left $d$-regular mask $\mathbf{M}$, and let $\beta = \beta(\mathbf{W})$. Define the cumulative coherence
\[
    \mu_1(s; \mathbf{W})
    \;:=\;
    \max_{j \in [n]}\;
    \max_{\substack{T \subset [n] \setminus \{j\} \\ |T| = s}}\;
    \sum_{\ell \in T}
    |\langle \mathbf{w}_j, \mathbf{w}_\ell \rangle|.
\]
If $\mathbf{M}$ is a $(k+1, \varepsilon, d)$-expander, then
\[
    \mu_1(s; \mathbf{W}) \;\le\; \beta^2 \varepsilon (s + 1)
    \qquad \text{for every } s \le k.
\]
In particular, if
\begin{equation}
\label{eq:omp_coherence_condition}
    \beta^2 \varepsilon (2k + 1) \;<\; 1,
\end{equation}
then $\mu_1(k; \mathbf{W}) + \mu_1(k-1; \mathbf{W}) < 1$, so the standard cumulative-coherence exact-recovery condition for OMP applies~\citep{tropp2004greed}. In the noiseless model $\mathbf{h} = \mathbf{b}_{\mathrm{dec}} + \mathbf{W}\mathbf{x}_\star$, OMP then recovers the support of every $k$-sparse $\mathbf{x}_\star$ in $k$ steps. This is the statement quoted in the main text as Corollary~\ref{cor:omp_main}.
\end{corollary}

\begin{proof}
For two columns $j \neq \ell$, let
\[
    c_{j\ell} \;=\; |\operatorname{supp}(\mathbf{w}_j) \cap \operatorname{supp}(\mathbf{w}_\ell)|.
\]
Since every nonzero decoder entry has magnitude at most $\beta / \sqrt{d}$,
\[
    |\langle \mathbf{w}_j, \mathbf{w}_\ell \rangle|
    \;\le\;
    \frac{\beta^2}{d}\, c_{j\ell}.
\]
Fix $j$ and $T$ with $|T| = s \le k$, and order $\{j\} \cup T$ with $j$ first. By the prefix collision bound (Lemma~\ref{lem:prefix_collision_bound}, applied at $r = s + 1 \le k + 1$), the total number of collision edges created by the $s + 1$ columns is at most $\varepsilon d (s + 1)$. Since $\sum_{\ell \in T} c_{j\ell}$ is bounded by this collision count,
\[
    \sum_{\ell \in T} |\langle \mathbf{w}_j, \mathbf{w}_\ell \rangle|
    \;\le\;
    \beta^2 \varepsilon (s + 1).
\]
Taking the maximum over $j$ and $T$ gives the cumulative-coherence bound, and the final statement follows from the standard OMP sufficient condition $\mu_1(k; \mathbf{W}) + \mu_1(k - 1; \mathbf{W}) < 1$~\citep{tropp2004greed}.
\end{proof}

The third corollary upgrades identifiability to a quantitative stability statement. Under the assumptions of Theorem~\ref{thm:expander_sae_recovery} the lower bound \eqref{eq:weighted_expander_lower_bound} holds with constant $c = \sqrt{d}(1/\beta - 2\beta\varepsilon) > 0$, and we show that this constant directly controls how the $\ell_1$-distance between two candidate $k$-sparse codes scales with the $\ell_1$-residual between their reconstructions.

\begin{corollary}[Stability under bounded residual mismatch]
\label{cor:stability}
Let $c = \sqrt{d}(1/\beta - 2\beta\varepsilon) > 0$ be the constant from Theorem~\ref{thm:expander_sae_recovery}. Under the assumptions of that theorem, suppose $\mathbf{x}_\star$ is $k$-sparse and $\mathbf{z}$ is $k$-sparse with
\[
    \|\mathbf{W}\mathbf{z} - \mathbf{W}\mathbf{x}_\star\|_1 \;\le\; \eta.
\]
Then
\[
    \|\mathbf{z} - \mathbf{x}_\star\|_1 \;\le\; \eta / c.
\]
\end{corollary}

\begin{proof}
The vector $\mathbf{u} = \mathbf{z} - \mathbf{x}_\star$ is $2k$-sparse, so the lower bound \eqref{eq:weighted_expander_lower_bound} gives $\|\mathbf{W}\mathbf{u}\|_1 \ge c\|\mathbf{u}\|_1$. Combining with $\|\mathbf{W}\mathbf{u}\|_1 \le \eta$ yields $\|\mathbf{u}\|_1 \le \eta / c$.
\end{proof}

\end{document}